\pdfoutput=1
\documentclass{article}
\PassOptionsToPackage{numbers}{natbib}

\usepackage[final]{neurips_2022}

\usepackage[utf8]{inputenc} 
\usepackage[T1]{fontenc}    
\usepackage{hyperref}       
\usepackage{url}            
\usepackage{booktabs}       
\usepackage{amsfonts}       
\usepackage{nicefrac}       
\usepackage{microtype}      
\usepackage{xcolor}         
\usepackage{amsmath}
\usepackage{amssymb}
\usepackage{amsthm}
\usepackage{bm}
\usepackage{graphicx}
\usepackage[group-separator={,}]{siunitx}
\usepackage{subfigure}
\DeclareMathOperator{\sech}{sech}

\newcommand{\vx}{\bm{x}}
\newcommand{\vu}{\bm{u}}
\newcommand{\vp}{\bm{p}}
\newcommand{\vn}{\bm{n}}

\newcommand{\vc}{\bm{c}}
\newcommand{\vw}{\bm{w}}
\newcommand{\vb}{\bm{b}}
\newcommand{\vg}{\bm{\sigma}}
\newcommand*\diff{\mathop{}\!\mathrm{d}}
\newcommand*\Diff[1]{\mathop{}\!\mathrm{d}^{#1}}

\newcommand{\modify}[1]{{#1}}

\theoremstyle{plain}
\newtheorem{theorem}{Theorem}[section]

\newtheorem{lemma}[theorem]{Lemma}

\theoremstyle{definition}

\newtheorem{assumption}[theorem]{Assumption}
\theoremstyle{remark}

\title{A Unified Hard-Constraint Framework for\\ Solving Geometrically Complex PDEs}

\author{%
  Songming Liu$^{1}$,~
  Zhongkai Hao$^{1}$,~ 
  Chengyang Ying$^{1}$,~
  Hang Su$^{1,2}$\thanks{Corresponding author}, ~
  Jun Zhu$^{1,2*}$, ~
  Ze Cheng$^{3}$
   \\
  $^{1}$Dept. of Comp. Sci. and Tech., Institute for AI, THBI Lab, BNRist Center, \\ Tsinghua-Bosch Joint ML Center, Tsinghua University\\
  $^{2}$Peng Cheng Laboratory; Pazhou Laboratory (Huangpu), Guangzhou, China\\
  $^{3}$Bosch Center for Artificial Intelligence\\
  \texttt{csuastt@gmail.com}
}

\begin{document}

\maketitle

\begin{abstract}
  We present a unified hard-constraint framework for solving geometrically complex PDEs with neural networks, where the most commonly used Dirichlet, Neumann, and Robin boundary conditions (BCs) are considered. Specifically, we first introduce the ``extra fields'' from the mixed finite element method to reformulate the PDEs so as to equivalently transform the three types of BCs into linear equations. Based on the reformulation, we derive the general solutions of the BCs analytically, which are employed to construct an ansatz that automatically satisfies the BCs. With such a framework, we can train the neural networks without adding extra loss terms and thus efficiently handle geometrically \modify{complex} PDEs, alleviating the unbalanced competition between the loss terms corresponding to the BCs and PDEs. We theoretically demonstrate that the ``extra fields'' can stabilize the training process. Experimental results on real-world geometrically complex PDEs showcase the \modify{effectiveness} of our method compared with state-of-the-art baselines.
\end{abstract}

\section{Introduction}

Many fundamental problems in science and engineering (e.g., \cite{batchelor2000introduction,majda2003introduction,schiesser2014computational}) are characterized by partial differential equations (PDEs) with the solution constrained by boundary conditions (BCs) that are derived from the physical system of the problem. Among all types of BCs, Dirichlet, Neumann, and Robin are the most commonly used \citep{strauss2007partial}. Figure~\ref{fig:three_type_bc} gives an illustrative example of these three types of BCs. Furthermore, in practical problems, physical systems can be very geometrically complex (\modify{where the geometry of the definition domain is irregular or has complex structures,} e.g., a lithium-ion battery \citep{jeon2011thermal}, a heat sink \citep{wu2017design}, etc), leading to a large number of BCs. How to solve such PDEs has become a challenging problem shared by both scientific and industrial communities.

\modify{The field of solving PDEs with neural networks has a history of more than 20 years \citep{dissanayake1994neural, bar2019unsupervised, sun2019solving, esmaeilzadeh2020meshfreeflownet, van2021optimally}}. Such methods are intrinsically mesh-free and therefore can handle high-dimensional as well as geometrically complex problems more efficiently compared with traditional mesh-based methods, like the finite element method (FEM). Physical-informed neural network (PINN) \citep{raissi2019physics} is one of the most influential works, where the neural network is trained in the way of taking the residuals of both PDEs and BCs as multiple terms of the loss function. \modify{Although there are many wonderful improvements such as DPM \cite{kim2020dpm}, PINNs still face serious challenges as discussed in the paper} \citep{krishnapriyan2021characterizing}. Some theoretical works \citep{wang2021understanding,wang2022and} point out that there exists an unbalanced competition between the terms of PDEs and BCs, limiting the application of PINNs to geometrically complex problems. To address this issue, some researchers \citep{berg2018unified,sun2020surrogate,lu2021physics} have tried to embed BCs into the ansatz. Some of them \citep{leake2020deep, schiassi2021extreme} follow the pipeline of the Theory of Connections \citep{mortari2017theory}, while others \citep{yu2018deep,kharazmi2019variational,liao2019deep} have considered solving the equivalent variational form of the PDEs. In this way, the neural networks can automatically satisfy the BCs and no longer require adding corresponding loss terms. Nevertheless, these methods are only applicable to specific BCs (e.g., Dirichlet BCs, homogeneous BCs, etc) or geometrically simple PDEs. The key challenge is that the equation forms of the Neumann and Robin BCs have no analytical solutions in general and are thus difficult to be embedded into the ansatz.

\begin{figure}[t!]
\begin{center}
\centerline{\includegraphics[width=.9\columnwidth]{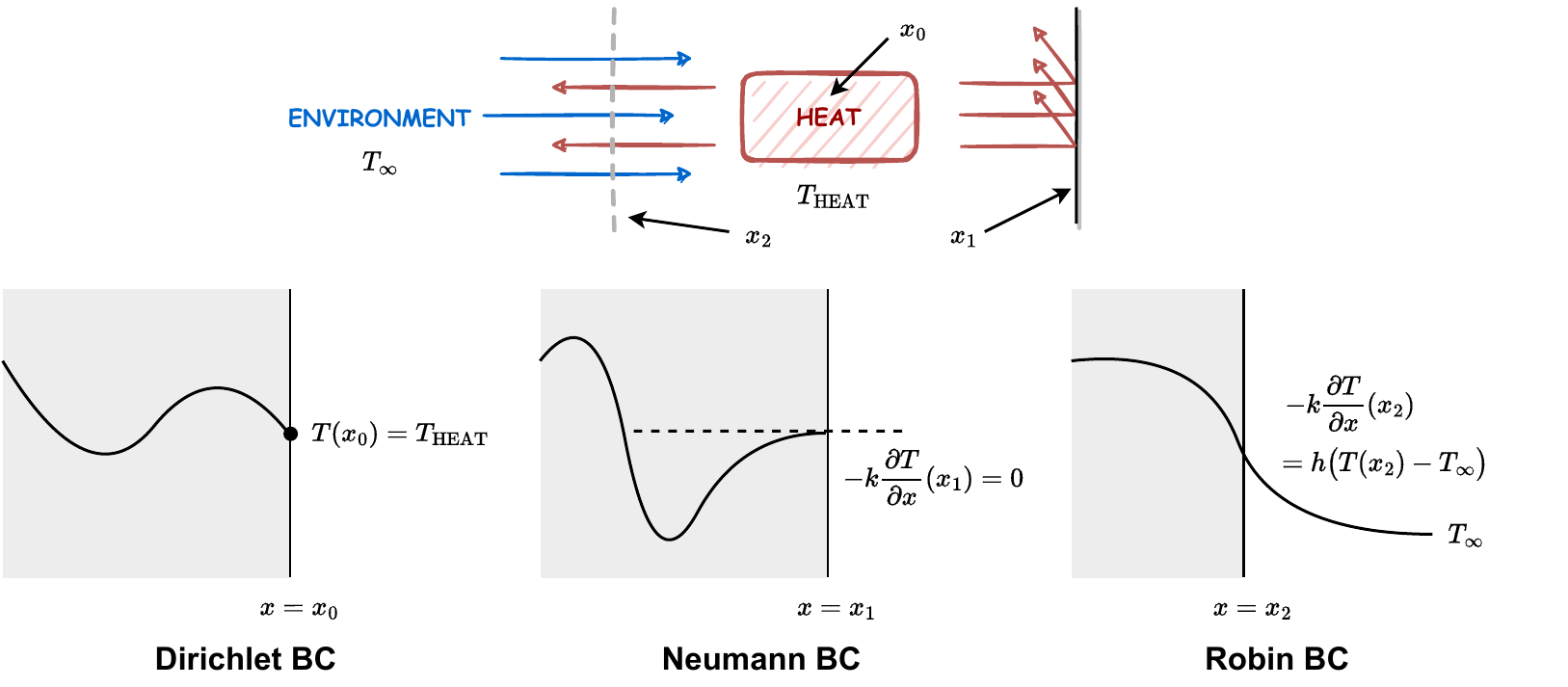}}
\caption{\textbf{An illustration on three types of BCs.} We give an example from heat transfer, where $T=T(x)$ is the temperature, $k$ is the thermal conductivity, and $h$ is the heat transfer coefficient. \textbf{(1)} The Dirichlet BC specifies the value of the solution at the boundary. Here we assume a constant temperature $T_{\mathrm{HEAT}}$ at the heat source ($x=x_0$). \textbf{(2)} The Neumann BC specifies the value of the derivative at the boundary. We assume that the right wall ($x=x_2$) is adiabatic and impose a zero normal derivative. \textbf{(3)} The Robin BC is a combination of the first two. And we use it to describe the heat convection between the heat source and the environment ($T_{\infty}$) at the surface ($x=x_2$).}
\label{fig:three_type_bc}
\end{center}
\vskip -0.3in
\end{figure}

In this paper, we propose a unified hard-constraint framework for all the three most commonly used BCs (i.e., Dirichlet, Neumann, and Robin BCs). With this framework (see Figure~\ref{fig:pipeline} for an illustration), we are able to construct an ansatz that automatically satisfies the three types of BCs. Therefore, we can train the model without the losses of these BCs, which alleviates the unbalanced competition and significantly improves the performance of solving geometrically complex PDEs. Specifically, we first introduce the \textit{extra fields} from the mixed finite element method \citep{malkus1978mixed,brink1996some}. This technique substitutes the gradient of a physical quantity with new variables, allowing the BCs to be reformulated as linear equations. Based on this reformulation, we derive a general continuous solution of the BCs of simple form, overcoming the challenge that the original BCs cannot be solved analytically. Using the general solutions obtained, we summarize a paradigm for constructing the hard-constraint ansatz under time-dependent, multi-boundary, and high-dimensional cases. Besides, in Section~\ref{sec:theo}, we demonstrate that the technique of \textit{extra fields} can improve the stability of the training process.

We empirically demonstrate the effectiveness of our method through three parts of experiments. Firstly, we show the potency of our method in solving geometrically complex PDEs through two numerical experiments from real-world physical systems of a battery pack and an airfoil. And our framework achieves a supreme performance compared with advanced baselines, including the learning rate annealing \modify{methods} \citep{wang2021understanding}, \modify{domain decomposition-based methods} \citep{jagtap2020extended, moseley2021finite}, and existing hard-constraint methods \citep{sheng2021pfnn, sheng2022pfnn}. Second, we select a high-dimensional problem to demonstrate that our framework can be well applied to \modify{high-dimensional} cases. Finally, we study the impact of the \textit{extra fields} as well as some hyper-parameters and verify our theoretical results in Section~\ref{sec:theo}.

To sum up, we make the following contributions: \begin{itemize}
    \item We introduce the \textit{extra fields} to reformulate the PDEs, and theoretically demonstrate that our reformulation can effectively reduce the instability of the training process.
    \item We propose a unified hard-constraint framework for Dirichlet, Neumann, and Robin BCs, \modify{alleviating} the unbalanced competition between losses in physics-informed learning.
    \item Our method has superior performance over state-of-the-art baselines on solving geometrically complex PDEs, as validated by numerical experiments in real-world physical problems.
\end{itemize}

\begin{figure}[t!]
\begin{center}
\centerline{\includegraphics[width=.9\columnwidth]{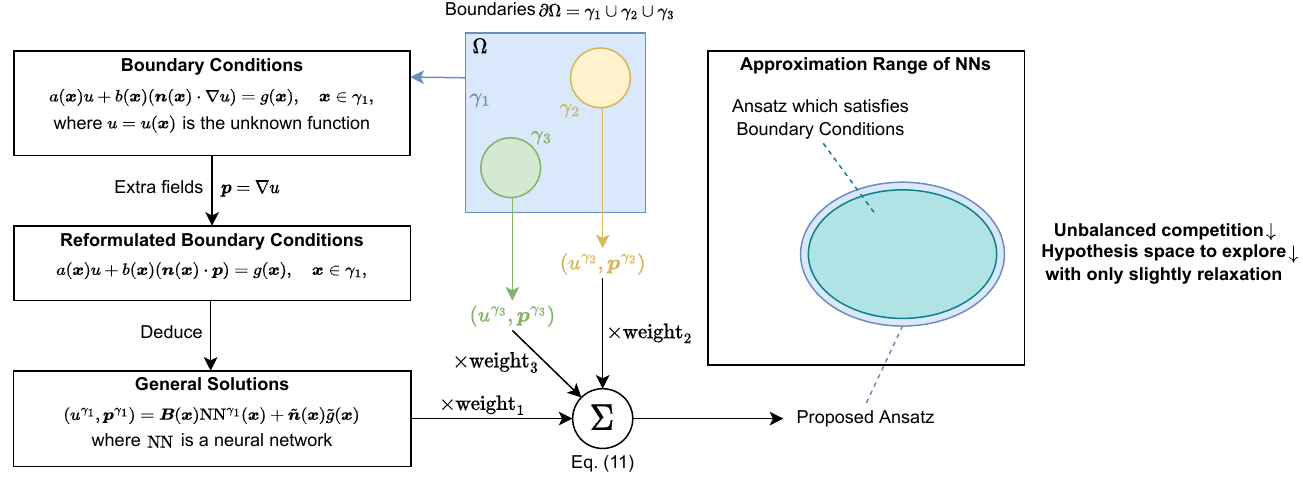}}
\caption{\textbf{A pipeline of the proposed method.} In this paper, we consider PDEs with multiple boundary conditions (BCs) of Dirichlet, Neumann, and Robin. We first introduce the \textit{extra fields} to reformulate the BCs as linear equations whose general solutions are deduced. Then, we aggregate the general solutions for each boundary to obtain our ansatz via Eq.~\eqref{eq_unified_hc}. Since the ansatz automatically satisfies the BCs, we alleviate the unbalanced competition and reduce invalid hypothesis space.}
\label{fig:pipeline}
\end{center}
\vskip -0.3in
\end{figure}

\section{Background}
\subsection{Physics-Informed Neural Networks (PINNs)}
\label{sec:pinn}
We consider the following Laplace's equation as a motivating example:
\begin{subequations}
\label{eq_toy_example}
\begin{align}
    \Delta u(x_1,x_2) &= 0, &&\modify{x_1\in(0,1],x_2\in[0,1]}, \label{eq_pde_toy_example}\\
    \modify{u(x_1,x_2)} &= g(x_2), &&\modify{x_1=0,x_2\in[0,1]}, \label{eq_bc_toy_example}
\end{align}
\end{subequations}
where $g\colon \mathbb{R}\rightarrow \mathbb{R}$ is a known fixed function, Eq.~\eqref{eq_pde_toy_example} gives the form of the PDE, and Eq.~\eqref{eq_bc_toy_example} is a Dirichlet boundary condition (BC). A solution to the above problem is a solution to Eq.~\eqref{eq_pde_toy_example} which also satisfies Eq.~\eqref{eq_bc_toy_example}. 

Physics-informed neural networks (PINNs)~\citep{raissi2019physics} employ a neural network $\mathrm{NN}(x_1,x_2;\bm{\theta})$ to approximate the solution, i.e., $\hat{u}(x_1,x_2;\bm{\theta})=\mathrm{NN}(x_1,x_2;\bm{\theta})\approx u(x_1,x_2)$, where $\bm{\theta}$ denotes the trainable parameters of the network. And we learn the parameters $\bm{\theta}$ by minimizing the following loss function:
\begin{equation}
    \mathcal{L}(\bm{\theta})=\mathcal{L}_{\mathcal{F}}(\bm{\theta}) + \mathcal{L}_{\mathcal{B}}(\bm{\theta})\triangleq \frac{1}{N_f}\sum_{i=1}^{N_f} \left| \Delta \hat{u}(x^{(i)}_{f,1},x^{(i)}_{f,2};\bm{\theta}) \right|^2 + \frac{1}{N_b}\sum_{i=1}^{N_b} \left| \hat{u}(0,x^{(i)}_{b,2};\bm{\theta}) - g(x^{(i)}_{b,2}) \right|^2,
    \label{eq:toy_example_loss}
\end{equation}
where $\mathcal{L}_{\mathcal{F}}$ is the term restricting $\hat{u}$ to satisfy the PDE (Eq.~\eqref{eq_pde_toy_example}) while $\mathcal{L}_{\mathcal{B}}$ is the one for the BC (Eq.~\eqref{eq_bc_toy_example}), $\{\vx^{(i)}_f=(x^{(i)}_{f,1},x^{(i)}_{f,2})\}_{i=1}^{N_f}$ is a set of $N_f$ collocation points sampled in $[0,1]^2$, and $\{\vx^{(i)}_b=(0,x^{(i)}_{b,2})\}_{i=1}^{N_b}$ is a set of $N_b$ boundary points sampled in $x_1=0\land x_2\in[0,1]$.

PINNs have a wide range of applications, including heat \citep{cai2021physics}, flow \citep{mao2020physics}, and atmosphere \citep{zhang2021spatiotemporal}. However, PINNs are struggling with some issues on the performance \citep{krishnapriyan2021characterizing}. Previous analysis~\citep{wang2021understanding,wang2022and} has demonstrated that the convergence of $\mathcal{L}_{\mathcal{F}}$ can be significantly faster than that of $\mathcal{L}_{\mathcal{B}}$. This pathology may lead to nonphysical solutions which do not satisfy the BCs or initial conditions (ICs). Moreover, for geometrically complex PDEs where the number of BCs is large, this problem is exacerbated and can seriously affect accuracy, as supported by our experimental results in Table~\ref{tab:pack}. 

\subsection{Hard-Constraint Methods}
One potential approach to overcome this pathology is to embed the BCs into the ansatz in a way that any instance from the ansatz can automatically satisfy the BCs, as utilized by previous works \citep{berg2018unified, pang2019fpinns,sun2020surrogate, wang2021cenn, sheng2021pfnn, lu2021physics, sheng2022pfnn}. We note that the loss terms corresponding to the BCs are no longer needed, and thus the above pathology is alleviated. These methods are called \textit{hard-constraint methods}, and they share a similar formula of the ansatz as:
\begin{equation}
\label{eq_general_formula_hc}
    \hat u(\vx;\bm{\theta})=u^{\partial \Omega}(\vx)+ l^{\partial \Omega}(\vx) \mathrm{NN}(\vx;\bm{\theta}) ,
\end{equation}
where $\vx$ is the coordinate, $\Omega$ is the domain of interest, $u^{\partial \Omega}(\vx)$ is the general solution at the boundary $\partial \Omega$, and $l^{\partial \Omega}(\vx)$ is an extended distance function which satisfies:
\begin{equation}\label{eq_ext_dist}
l^{\partial \Omega}(\vx)=
    \begin{cases}
    0& \mathrm{if}\ \vx\in\partial\Omega,\\
    >0& \mathrm{otherwise}.
    \end{cases}
\end{equation}

In the case of Eq.~\eqref{eq_toy_example} (where $\vx=(x_1,x_2)$, $\Omega=[0,1]^2$), the general solution is exactly $g(x_2)$, and we can use the following ansatz (which automatically satisfies the BC in Eq.~\eqref{eq_bc_toy_example}):
\begin{equation}
    \hat u(x_1,x_2;\bm{\theta})= g(x_2) + x_1 \mathrm{NN}(x_1,x_2;\bm{\theta}).
\end{equation}

However, it is hard to directly extend this method to more general cases of Robin BCs (see Eq.~\eqref{eq:general_bcs}), since we cannot obtain the general solution $u^{\partial \Omega}(\vx)$ analytically. Existing attempts  are either mesh-dependent \citep{gao2021phygeonet, zhao2021physics}, which are time-consuming for high-dimensional and geometrically complex PDEs, or ad hoc methods for specific (geometrically simple) physical systems \citep{rao2021physics}. It is still lacking a unified hard-constraint framework for both geometrically complex PDEs and the most commonly used Dirichlet, Neumann, and Robin BCs.

\section{Methodology}
We first introduce the problem setup of geometrically complex PDEs considered in this paper and then reformulate the PDEs via the \textit{extra fields}, followed by presenting our unified hard-constraint framework for embedding Dirichlet, Neumann, and Robin BCs into the ansatz.

\subsection{Problem Setup}
We consider a physical system governed by the following PDEs defined on a geometrically complex domain: $\Omega\subset\mathbb{R}^d$
\begin{equation}
    \mathcal{F}[\vu(\vx)]=\bm{0},\qquad\vx=(x_1,\dots,x_d)\in\Omega,
\label{eq:genearl_PDE}
\end{equation}
where $\mathcal{F}=(\mathcal{F}_1,\dots,\mathcal{F}_N)$ includes $N$ PDE operators which map $\vu$ to a function of $\vx$, $\vu$ and $\vu$'s derivatives. Here, $\vu(\vx)=(u_1(\vx),\dots,u_n(\vx))$ is the unknown solution, which represents physical quantities of interest. For each $u_j, j=1,\dots,n$, we impose suitable boundary conditions (BCs) as:
\begin{equation}
    a_{j,i}(\vx)u_j+b_{j,i}(\vx)\big( \vn_{j,i}(\vx)\cdot\nabla u_j \big) = g_{j,i}(\vx),\quad\vx\in\gamma_{j,i},\quad \forall i=1,\dots,m_j,
\label{eq:general_bcs}
\end{equation}
where $\{\gamma_{j,i}\}_{i=1}^{m_j}$ are subsets of the boundary $\partial\Omega$ whose closures are disjoint, $a_{j,i}, b_{j,i}\colon \gamma_{j,i} \rightarrow \mathbb{R}$ satisfy that $a_{j,i}^2(\vx)+b^2_{j,i}(\vx)\neq 0, \forall \vx\in\gamma_{j,i}$, $\vn_{j,i}\colon \gamma_{j,i} \rightarrow \mathbb{R}^d$ is the (outward facing) unit normal of $\gamma_{j,i}$ at each location, and $g_{j,i}\colon \gamma_{j,i} \rightarrow \mathbb{R}$. It is noted that Eq.~\eqref{eq:general_bcs} represents a Dirichlet BC if $a_{j, i}\equiv 1,b_{j, i}\equiv0$, a Neumann BC if $a_{j, i}\equiv 0,b_{j, i}\equiv 1$, and a Robin BC otherwise. In the following, we drop some of the subscripts in Eq.~\eqref{eq:general_bcs} for clarity \footnote{We simplify $\gamma_{j, i}, a_{j, i}(\vx), b_{j, i}(\vx), \vn_{j, i}(\vx), g_{j, i}(\vx)$ to $\gamma_{i}, a_{i}(\vx), b_{i}(\vx), \vn(\vx), g_{i}(\vx)$.}.

For such geometrically complex PDEs, if we directly resort to PINNs (see Section~\ref{sec:pinn}), there would be a difficult multi-task learning with at least $(\sum_{j=1}^n m_j + N)$ terms in the loss function. As discussed in the previous analyses \citep{wang2021understanding, wang2022and}, it will severely affect the convergence of the training due to the unbalanced competition between those loss terms. Hence, in this paper, we will discuss how to embed the BCs into the ansatz, where every instance automatically satisfies the BCs. However, it is infeasible to directly follow the pipeline of \textit{hard-constraint methods} (see Eq.~\eqref{eq_general_formula_hc}) since Eq.~\eqref{eq:general_bcs} does not have a general solution of analytical form. Therefore, a new approach is needed to address this intractable problem.

\subsection{Reformulating PDEs via Extra Fields}
\label{sec:extra_fields}
In this subsection, we present the general solutions of the BCs, which will be used to construct the hard-constraint ansatz subsequently. We first introduce the \textit{extra fields} from the mixed finite element method \citep{malkus1978mixed,brink1996some} to equivalently reformulate the PDEs. Let $\vp_j(\vx)=(p_{j1}(\vx),\dots,p_{jd}(\vx))=\nabla u_j,j=1,\dots,n$. We substitute them into Eq.~\eqref{eq:genearl_PDE} and Eq.~\eqref{eq:general_bcs} to obtain the equivalent PDEs:
\begin{subequations}
\label{eq:pde_after_extra_field}
\begin{align}
    \tilde{\mathcal{F}}[\vu(\vx),\vp_1(\vx),\dots,\vp_n(\vx)]&=\bm{0},&&\vx\in\Omega,\\
    \vp_j(\vx)&=\nabla u_j,&&\vx\in\Omega\cup\partial\Omega,&&\forall j=1,\dots,n,
    \label{eq:balance_eq_after_extra_field}
\end{align}
\end{subequations}
where $(\vu(\vx),\vp_1(\vx),\dots,\vp_n(\vx))$ is the solution to the new PDEs, $\tilde{\mathcal{F}}=(\tilde{\mathcal{F}}_1,\dots,\tilde{\mathcal{F}}_N)$ are the PDE operators after the reformulation. And for $j=1,\dots,n$, we have the corresponding BCs:
\begin{equation}
    a_i(\vx)u_j+b_i(\vx)\big( \vn(\vx)\cdot\vp_j(\vx) \big) = g_i(\vx),\qquad\vx\in\gamma_i,\qquad \forall i=1,\dots,m_j.
    \label{eq:bc_after_extra_field}
\end{equation}

Now, we can see that Eq.~\eqref{eq:general_bcs} has been transformed into linear equations with respect to $(u_j, \vp_j)$, which are much easier for us to derive general solutions. Hereinafter, we denote $(u_j,\vp_j)$ by $\tilde{\vp}_j$, and the general solution to the BC at $\gamma_i$ by $\tilde{\vp}_j^{\gamma_i}=(u_j^{\gamma_i}, \vp_j^{\gamma_i})$. Next, we will discuss how to 
obtain $\tilde{\vp}_j^{\gamma_i}$.


To obtain the general solution of Eq. \eqref{eq:bc_after_extra_field}, the first step is to find a basis $\bm{B}(\vx)$ of the null space (whose dimension is $d$), which should include $d$ vectors. However, we must emphasize that $\bm{B}(\vx)$ should be carefully chosen. Since Eq.~\eqref{eq:bc_after_extra_field} is parameterized by $\vx$, for any $\vx\in\gamma_i$, $\bm{B}(\vx)$ should always be a basis of the null space, that is, its columns cannot degenerate into linearly dependent vectors (otherwise it will not be able to represent all possible solutions). An example of an inadmissible $\bm{B}(\vx)$ is given in Appendix~\ref{sec_a1}.

Generally, for any dimension $d\in\mathbb{N}^+$, \modify{we believe it is non-trivial to find a simple expression for the basis}. Instead, we prefer to find $(d+1)$ vectors in the null space, $d$ of which are linearly independent (that way, $\bm{B}(\vx)\in\mathbb{R}^{(d+1) \times (d+1)}$). \modify{We now directly present our construction of the general solution while leaving a detailed derivation in Appendix~\ref{sec_a3}.}
\begin{equation} \label{eq_hc_robin}
    \tilde{\vp}_j^{\gamma_i}\modify{(\vx;\bm{\theta}_j^{\gamma_i})}=\bm{B}(\vx)\modify{\mathrm{NN}_j^{\gamma_i}(\vx;\bm{\theta}_j^{\gamma_i})} + \tilde{\vn}(\vx)\tilde{g}_i(\vx),
\end{equation}
where $\tilde{\vn}=(a_i,b_i\vn)\big/\sqrt{a_i^2+b_i^2}$, $\tilde{g}_i=g_i\big/\sqrt{a_i^2+b_i^2}$, $\modify{\mathrm{NN}_j^{\gamma_i}}:\mathbb{R}^{d}\to\mathbb{R}^{d+1}$ is a neural network\modify{ with trainable parameters $\bm{\theta}_j^{\gamma_i}$}, and $\bm{B}(\vx)=\bm{I}_{d+1}-\tilde{\vn}(\vx){\tilde{\vn}(\vx)}^\top$ is the basis we have found (precisely, it is a set of vectors that always contain a basis of the null space for any $\vx\in\gamma_i$). Incidentally, in the case of $d=1$ or $d=2$, we can find a simpler expression for $\bm{B}$ \modify{(see Appendix~\ref{sec_a2})}. We note that there is no restriction for the architecture of neural networks used and MLPs are chosen as default.

\subsection{A Unified Hard-Constraint Framework}\label{sec:unified_hc}
With the parameterization of a neural network, Eq.~\eqref{eq_hc_robin} can represent any function defined on $\gamma_i$, as long as the function satisfies the BC (see Eq.~\eqref{eq:bc_after_extra_field}). Since our problem domain contains multiple boundaries, we need to combine the general solutions corresponding to each boundary $\gamma_i$ to achieve an overall approximation. Hence, we construct our ansatz (for each $1\le j \le n$ \emph{separately}) as follows:
\begin{equation} \label{eq_unified_hc}
    (\hat{u}_j,\hat{\vp}_j) = l^{\partial \Omega}(\vx) \modify{\mathrm{NN}^{\mathrm{main}}_j(\vx;\bm{\theta}^{\mathrm{main}}_j)} + \sum_{i=1}^{m_j} \exp{\big[-\alpha_i l^{\gamma_i}(\vx)\big]}\tilde{\vp}_j^{\gamma_i}(\vx\modify{;\bm{\theta}_j^{\gamma_i}}),\quad \forall j=1,\dots,n,
\end{equation}
where $\modify{\mathrm{NN}^{\mathrm{main}}_j}:\mathbb{R}^{d}\to\mathbb{R}^{d+1}$ \modify{is the main neural network with trainable parameters $\bm{\theta}^{\mathrm{main}}_j$}, $l^{\partial \Omega}, l^{\gamma_i},i=1,\dots,m_j$ are \modify{continuous} extended distance functions (similar to Eq.~\eqref{eq_ext_dist}), and $\alpha_i~(i=1,\dots,m_j)$ are determined by:
\begin{equation}\label{eq_determine_alpha}
\alpha_i = \frac{\beta_s}{\min_{\vx \in \partial\Omega \setminus \gamma_i} l^{\gamma_i}(\vx) } ,
\end{equation}
where $\beta_s\in\mathbb{R}$ is a hyper-parameter of the ``hardness'' in the spatial domain. \modify{In Eq.~\eqref{eq_unified_hc}, we utilize extended distance functions to ``divide'' the problem domain into several parts, where $\{\tilde{\vp}_j^{\gamma_i}\}_{i=1}^{m_j}$ is responsible for the approximation on the boundaries while $\mathrm{NN}^{\mathrm{main}}_j$ are responsible for internal. Furthermore, Eq.~\eqref{eq_determine_alpha} ensures that the importance of  $\tilde{\vp}_j^{\gamma_i}$ decays to $e^{-\beta_s}$ on the nearest boundary from $\gamma_i$, so that $\tilde{\vp}_j^{\gamma_i}$ does not appear on other boundaries. We provide a theoretical guarantee 
for the correctness and approximation ability of Eq.~\eqref{eq_unified_hc} in Appendix~\ref{sec_a4}.} Besides, if $a_i$, $b_i$, $\vn$ or $g_i$ are only defined at $\gamma_i$, we can extend their definition to $\Omega\cup\partial\Omega$ using interpolation techniques or neural networks (see Appendix~\ref{sec_a5}). An extension of our framework to the temporal domain is discussed in Appendix~\ref{sec_a6}. 


Finally, we can train our model with the following loss function:
\begin{equation}
\begin{aligned}
    \mathcal{L}&= \frac{1}{N_f}\sum_{k=1}^{N_f}\sum_{j=1}^N \big| \tilde{\mathcal{F}}_j[\hat{\vu}(\vx^{(k)}),\hat{\vp}_1(\vx^{(k)}),\dots,\hat{\vp}_n(\vx^{(k)})] \big|^2 \\
    &+ \frac{1}{N_f}\sum_{k=1}^{N_f}\sum_{j=1}^n \big\Vert \hat{\vp}_j(\vx^{(k)}) - \nabla\hat{u}_j(\vx^{(k)}) \big\Vert_2^2 ,
    \label{eq:hc_loss}
\end{aligned}
\end{equation}
where $\hat{\vu}=(\hat{u}_1,\dots,\hat{u}_n)$, $\{\hat{u}_j,\hat{\vp}_j\}_{j=1}^n$ is defined in Eq.~\eqref{eq_unified_hc} and $\{\vx^{(k)}\}_{k=1}^{N_f}$ is a set of collocation points sampled in $\Omega$. \modify{For neatness, we omit the trainable parameters of neural networks here. We note that} Eq.~\eqref{eq:hc_loss} measures the discrepancy of both the PDEs (i.e., $\tilde{\mathcal{F}}_1,\dots,\tilde{\mathcal{F}}_N$) and the equilibrium equations introduced by the \textit{extra fields} (i.e., Eq.~\eqref{eq:balance_eq_after_extra_field}) at $N_f$ collocation points.

According to Eq.~\eqref{eq:hc_loss}, we have now successfully embedded BCs into the ansatz, and no longer need to take the residuals of BCs as extra terms in the loss function. That is, our ansatz strictly conforms to BCs throughout the training process, greatly reducing the possibility of generating nonphysical solutions. Nevertheless, this comes at the cost of introducing $(nd)$ additional equilibrium equations (see Eq.~\eqref{eq:balance_eq_after_extra_field}). But in many physical systems, especially those with complex geometries, the number of BCs ($\mathrm{cnt}(\mathrm{BCs})$) is far larger than $(nd)$ (e.g., $n=3$, $d=2$, $\mathrm{cnt}(\mathrm{BCs})=1260$ for a classical physical system, a heat exchanger \citep{dubovsky2011analytical}). So we may actually reduce the number of loss terms by orders of magnitude ($\Delta\mathrm{cnt}(\mathrm{losses}) = n d - \mathrm{cnt}(\mathrm{BCs}) \ll 0$), alleviating the unbalanced competition between loss terms.

\section{Theoretical Analysis} \label{sec:theo}


In Section~\ref{sec:extra_fields}, we introduce the \textit{extra fields} and reformulate the PDEs (from Eq.~\eqref{eq:genearl_PDE} and Eq.~\eqref{eq:general_bcs} to Eq.~\eqref{eq:pde_after_extra_field} and Eq.~\eqref{eq:bc_after_extra_field}). To further analyze the impact of this reformulation, we consider the following abstraction of 1D PDEs (Eq.~\eqref{eq_theo}) and the one after the reformulation (Eq.~\eqref{eq_theo_extra_fields}):
\begin{subequations}
\begin{align}
    L\big[ u, \frac{\diff{u}}{\diff{x}}, \frac{\Diff{2}{u}}{\diff{x}^2},\cdots,\frac{\Diff{n}{u}}{\diff{x}^n} \big] = f(x),\qquad x \in \Omega,\label{eq_theo}\\
    L\big[ u, p_1, p_2, \cdots, p_{n-1}] 
    = f(x),\qquad x \in \Omega,\label{eq_theo_extra_fields}
\end{align}
\end{subequations}
where $L$ is an operator, $f$ is a source function, and $p_{i+1}= \diff{p_i}/\diff{x}, p_1 = \diff{u}/\diff{x}$ are the introduced extra fields. From the above equations, we can find that the orders of derivatives in the PDEs are reduced. Intuitively, as the PDEs are included in the loss function (see Eq.~\eqref{eq:toy_example_loss}), lower derivatives result in less numerical sensitivity and thus stabilize the training process. 

\begin{figure}[!b]
    \centering
    \subfigure[A 2D battery pack]{\label{fig:geom_heat}\includegraphics[width=.47\textwidth]{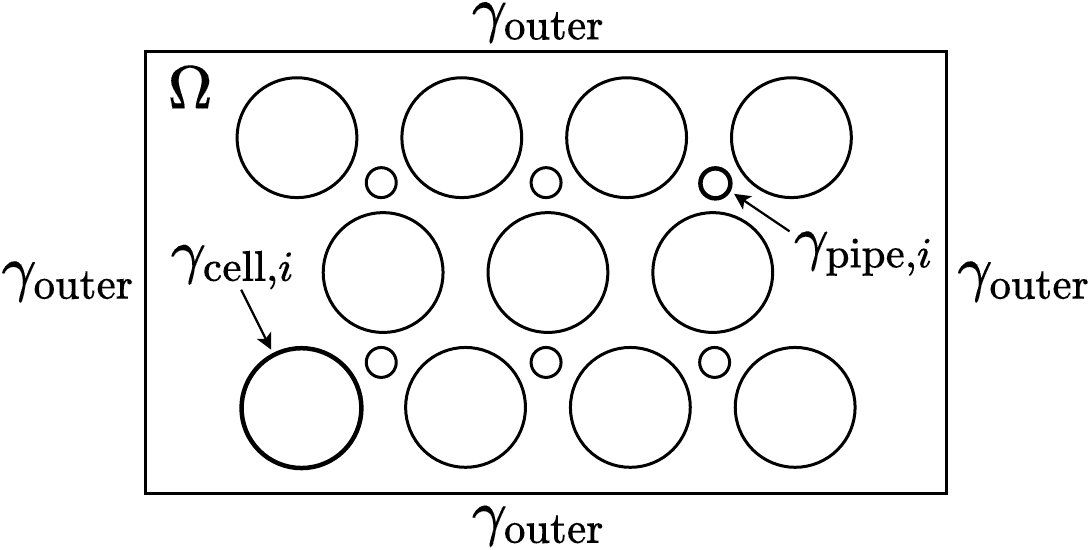}}
    \hspace{1em}
    \subfigure[An airfoil]{\label{fig:geom_airfoil}\includegraphics[width=.47\textwidth]{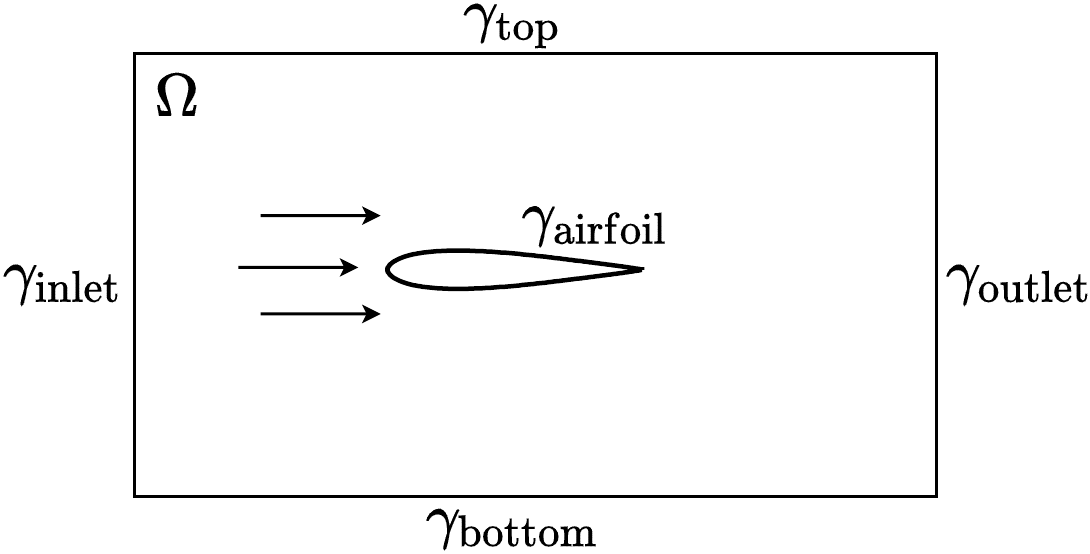}}
    \vspace{-0.1in}
    \caption{An illustration of the geometries of the problems in Section~\ref{sec:exp_heat} and Section~\ref{sec:exp_ns}.}
    \vspace{-0.1in}
\end{figure}

To formally explain this mechanism, without loss of generality, we consider a simple case (with proper Dirichlet BCs) where $L[\cdots]=\Delta u$, $f(x)=-a^2\sin{ax}$. And our ansatz is a single-layer neural network of \modify{width} $K$, i.e., $\hat{u}=\vc^\top \modify{\sigma}(\vw x+\vb)$, $\hat{p}=\vc_{p}^\top \modify{\sigma}(\vw x+\vb)$, where $\vc,\vw,\vb,\vc_{p}\in\mathbb{R}^K$, $\modify{\sigma}$ is an element-wise activation function (for simplicity, we take $\modify{\sigma}$ as $\tanh$). More details on the PDEs are given in Appendix~\ref{sec_a7_1}. Next, we focus on the gradients of the loss terms of the PDEs $\mathcal{F}$, since that of the BC stays the same during the reformulation. We state the following theorem.
\begin{theorem}[Bounds for the Gradients of the PDE Loss Terms] \label{theo:universal}
Let $\bm{\theta}=(\vc,\vw,\vb)$, $\tilde{\bm{\theta}}=(\vc,\vw,\vb,\vc_p)$, and $\mathcal{L}_{\mathcal{F}}$ as well as $\tilde{\mathcal{L}}_{\mathcal{F}}$ be the loss terms corresponding to the original and transformed PDEs, respectively. We have the following bounds:
\begin{subequations}\label{eq:theo_grad}
\begin{align}
    \Big|\big(\nabla_{\bm{\theta}}\mathcal{L}_{\mathcal{F}}\big)^\top\Big| &=  \mathcal{O}\big(|\vc|^\top\vw^2 + a^2\big)\cdot\big(\vw^2, |\vc| \circ |\vw|\circ( |\vw| + \bm{1} ), |\vc| \circ \vw^2\big), \label{eq_grad}\\
    \Big|\big({\nabla_{\tilde{\bm{\theta}}}\tilde{\mathcal{L}}_{\mathcal{F}}}\big)^\top\Big| &=  \mathcal{O}\big(\lVert\vc_p\rVert_1 + \max(|\vc|,|\vc_p|)^\top|\vw| + a^2\big)\cdot \big(
    |\vw|, \max(|\vc|,|\vc_p|)\circ\max(|\vw|, \bm{1}),\nonumber\\
    &\qquad\qquad\qquad\qquad\qquad\max(|\vc|,|\vc_p|)\circ\max(|\vw|,\bm{1}), \max(|\vw|,\bm{1})
    \big), \label{eq_grad_transformed}
\end{align}
\end{subequations}
where $\circ$ is the element-wise multiplication, $\bm{1}$ is an all-ones vector, $\max(\cdot)$ is the element-wise maximum of vectors, and operations on vectors ($\lvert\cdot\rvert$, $(\cdot)^2$, etc) are element-wise operations.
\end{theorem}
The proof is deferred to Appendix~\ref{sec_a7_2}. We can find that $(\nabla_{\bm{\theta}}\mathcal{L}_{\mathcal{F}})^\top$ has a high-order relationship with the \modify{trainable} parameters $\bm{\theta}$. This often makes training unstable, as the gradients $(\nabla_{\bm{\theta}}\mathcal{L}_{\mathcal{F}})^\top$ may rapidly explode or vanish when $\bm{\theta}$ are large or small. After the reformulation, $({\nabla_{\tilde{\bm{\theta}}}\tilde{\mathcal{L}}_{\mathcal{F}}})^\top$ is associated with the parameters in a lower order, and some terms are controlled by $1$, alleviating this issue of the instability. We note that the improvement in stability comes from lower derivatives in the PDEs, and this mechanism does not depend on the specific form of the PDEs. In Section~\ref{sec:exp_abla_extra}, we empirically show that this mechanism is general, even for nonlinear PDEs.

\section{Experiments}\label{sec:exp}
In this section, we will numerically showcase the effectiveness of the proposed framework on geometrically complex PDEs, which is also the main focus of this paper. To this end, we test our framework on two real-world problems (see Section~\ref{sec:exp_heat} and Section~\ref{sec:exp_ns}) from a 2D battery pack and an airfoil, respectively. Besides, to validate that our framework is applicable to high-dimensional cases, an \modify{additional} high-dimension problem is considered in Section~\ref{sec:exp_high}, followed by an ablation study in Section~\ref{sec:exp_abla}. We refer to Appendix~\ref{sec_a8} for experimental details.

\subsection{Experiment Setup}\label{sec:exp_setup}
\paragraph{Evaluation}
In our experiments, we utilize mean absolute error (MAE) and mean absolute percentage error (MAPE) to measure the discrepancy between the prediction and the ground truth. In Section~\ref{sec:exp_ns}, we replace MAPE with weighted mean absolute percentage error (WMAPE) to avoid the ``division
by zero'', since a considerable part of the ground truth values are very close to zero, 
\begin{equation}
    \mathrm{WMAPE} = \frac{\sum_{i=1}^n|\hat{y}_i-y_i|}{\sum_{i=1}^n|y_i|},
\end{equation}
where $n$ is the number of testing samples, $\hat{y}_i$ is the prediction and $y_i$ is the ground truth.

\paragraph{Baselines}
We consider the following baselines in subsequent experiments:
\begin{itemize}
    \item \textbf{PINN}: original implementation of the PINN \citep{raissi2019physics}.
    \item \textbf{PINN-LA \& PINN-LA-2}: PINN with learning rate annealing algorithm to address the unbalanced competition \citep{wang2021understanding} and a variant we modified for numerical stability.
    \item \textbf{xPINN\& FBPINN}: two representative works for solving geometrically complex PDEs via non-overlapping \citep{jagtap2020extended} and overlapping \citep{moseley2021finite} domain decomposition.
    \item \textbf{PFNN \& PFNN-2}: a representative hard-constraint method based on the variational formulation of PDEs \citep{sheng2021pfnn} and its latest variant incorporating domain decomposition \citep{sheng2022pfnn}.
\end{itemize}

\begin{table}[!b]
    \caption{Experimental results of the simulation of a 2D battery pack}
    \centering
    \begin{small}
    \begin{tabular}{lllllllll}
    \toprule
    & \multicolumn{4}{c}{MAE of $T$} & \multicolumn{4}{c}{MAPE of $T$}\\
    \cmidrule(r){2-5}
    \cmidrule(r){6-9}
    & $t=0$     & $t=0.5$     & $t=1$ & average & $t=0$     & $t=0.5$     & $t=1$ & average \\
    \midrule
    PINN & $0.1283$     & $0.0457$ & $0.0287$ & $0.0539$ & $128.21\%$ & $11.65\%$  & $4.47\%$ & $24.82\%$\\
    PINN-LA & $0.0918$     & $0.0652$  & $0.0621$ & $0.0661$ & $91.72\%$ & $19.13\%$  & $11.96\%$  & $27.06\%$\\
    PINN-LA-2 & $0.1062$     & $0.0321$  & $\textbf{0.0211}$ & $0.0402$ & $106.05\%$ & $8.94\%$  & $4.09\%$  & $19.76\%$\\
    FBPINN & $0.0704$     & $0.0293$& $0.0249$ & $0.0343$ & $70.33\%$ & $8.13\%$  & $5.87\%$ & $14.74\%$\\
    xPINN & $0.2230$     & $0.1295$  & $0.1515$ & $0.1454$ & $222.83\%$ & $30.28\%$  & $20.25\%$  & $54.70\%$\\
    PFNN & $\textbf{0.0000}$     & $0.3036$  & $0.4308$ & $0.2758$ & $\textbf{0.02\%}$ & $79.64\%$  & $84.60\%$  & $68.29\%$\\
    PFNN-2 & $\textbf{0.0000}$     & $0.3462$  & $0.5474$ & $0.3215$ & $\textbf{0.02\%}$ & $66.06\%$  & $90.21\%$  & $59.62\%$\\
    \midrule
    HC & $\textbf{0.0000}$ & $\textbf{0.0246}$  & $0.0225$ & $\textbf{0.0221}$ & $\textbf{0.02\%}$ & $\textbf{5.38\%}$  & $\textbf{3.77\%}$ & $\textbf{5.10\%}$\\
    \bottomrule
    \end{tabular}
    \end{small}
    \label{tab:pack}
\end{table}

\subsection{Simulation of a 2D battery pack (Heat Equation)} \label{sec:exp_heat}
To emphasize the capability of the proposed hard-constraint framework (HC) to handle geometrically complex cases, we first consider modeling thermal dynamics of a 2D battery pack \citep{smyshlyaev2011pde}, which is governed by the heat equation (see Appendix~\ref{sec_a8_1}), where $\vx\in\Omega, t\in[0,1]$ are the spatial and temporal coordinates, respectively, $T(\vx,t)$ is the temperature over time. The geometry of the problem $\Omega$ is shown in Figure~\ref{fig:geom_heat}, where $\gamma_{\mathrm{cell},i}$ is the cell and $\gamma_{\mathrm{pipe},i}$ is the cooling pipe.


The baselines are trained with $N_f=\num{8192}$ collocation points, $N_b=512$ boundary points, and $N_i=512$ initial points (an additional $N_s=512$ points on the interfaces between subdomains are required for the xPINN), while the HC is trained with $N_f=\num{8192}$ collocation points. All the models are trained for $\num{5000}$ Adam~\cite{kingma2015adam} iterations (with a learning rate scheduler) followed by an L-BFGS optimization until convergence and tested with $N_t=\num{146487}$ points evaluated by the FEM. The testing results are given in Table~\ref{tab:pack} (where ``average'' means averaging over $t\in[0,1]$, the same below). We find that the accuracy of the HC is significantly higher than baselines almost all the time, especially at $t=0$. This is because the ansatz of the HC always satisfies the BCs and ICs during the training process, preventing the approximations from violating physical constraints at the boundaries. Besides, we also provide an empirical analysis of convergence in Appendix~\ref{sec_a9} and results of 5-run parallel experiments (where problems in all three cases are revisited) in Appendix~\ref{sec_a10}.

\subsection{Simulation of an Airfoil (Navier-Stokes Equations)} \label{sec:exp_ns}
In the next experiment, we consider the challenging benchmark PDEs in computational fluid dynamics, the 2D stationary incompressible Navier-Stokes equations, in the context of simulating the airflow around a real-world airfoil (\texttt{w1015.dat}) from the UIUC airfoil coordinates database (an open airfoil database) \citep{uiucairfoil}. Specifically, the airfoil is represented by $\num{240}$ anchor points and the governing equation can be found in Appendix~\ref{sec_a8_2}, where $\vu(\vx) = (u_1(\vx), u_2(\vx))$, $p(\vx)$ are the velocity and pressure of the fluid, respectively, and the geometry of the problem is shown in Figure~\ref{fig:geom_airfoil}.

The baselines (PFNN and PFNN-2 are not applicable to Navier-Stokes Equations) are trained with $N_f=\num{10000}$ collocation points and $N_b=2048$ boundary points (additional $N_s=2048$ points for the xPINN), while the HC is trained with $N_f=\num{10000}$ collocation points. All the models are trained for $\num{5000}$ Adam iterations followed by an L-BFGS optimization until convergence and tested with $N_t=\num{13651}$ points evaluated by the FEM. According to the results in Table~\ref{tab:airfoil}, the HC has more significant advantages than the previous experiment. This is because complex nonlinear PDEs are more sensitive to the BCs, and the failure at the BCs can cause the approximation to deviate significantly from the true solution. In addition, the WMAPE of $u_2$ is relatively large for all methods, because the true values of $u_2$ are close to zero, amplifying the effect of noises.

\begin{table}[!t]
    \caption{Experimental results of the simulation of an airfoil}
    \centering
    \begin{small}
    \begin{tabular}{lllllll}
    \toprule
    & \multicolumn{3}{c}{MAE} & \multicolumn{3}{c}{WMAPE}\\
    \cmidrule(r){2-4}
    \cmidrule(r){5-7}
    & $u_1$     & $u_2$     & $p$ & $u_1$     & $u_2$     & $p$ \\
    \midrule
    PINN & $0.4682$     & $0.0697$ & $0.3883$ & $0.5924$  & $1.1979$ & $0.3539$\\
    PINN-LA & $0.4018$     & $0.0595$& $0.2652$ & $0.5084$  & $1.0225$ & $0.2418$\\
    PINN-LA-2 & $0.5047$     & $0.0659$  & $0.2765$ & $0.6385$  & $1.1325$ & $0.2521$\\
    FBPINN & $0.4058$     & $0.0563$& $0.2665$ & $0.5134$  & $0.9676$ & $0.2429$\\
    xPINN & $0.7188$     & $0.0583$& $1.1708$ & $0.9095$  & $1.0029$ & $1.0672$\\
    \midrule
    HC & $\textbf{0.2689}$ & $\textbf{0.0435}$  & $\textbf{0.2032}$ & $\textbf{0.3402}$  & $\textbf{0.7474}$ & $\textbf{0.1852}$\\
    \bottomrule
    \end{tabular}
    \end{small}
    \label{tab:airfoil}
    \vspace{-0.1in}
\end{table}

\begin{table}[!b]
    \vspace{-0.1in}
    \caption{Experimental results of the high-dimensional heat equation}
    \centering
    \begin{small}
    \begin{tabular}{lllllllll}
    \toprule
    & \multicolumn{4}{c}{MAE of $u$} & \multicolumn{4}{c}{MAPE of $u$}\\
    \cmidrule(r){2-5}
    \cmidrule(r){6-9}
    & $t=0$     & $t=0.5$     & $t=1$ & average & $t=0$     & $t=0.5$     & $t=1$ & average \\
    \midrule
    PINN & $0.0219$     & $0.0428$ & $0.1687$ & $0.0582$ & $1.43\%$ & $1.70\%$  & $4.04\%$ & $1.99\%$\\
    PINN-LA & $0.0085$     & $0.0149$  & $0.0727$ & $0.0235$ & $0.55\%$ & $0.59\%$  & $1.73\%$  & $0.78\%$\\
    PINN-LA-2 & $0.0122$     & $0.0274$  & $0.1495$ & $0.0466$ & $0.79\%$ & $1.08\%$  & $3.57\%$  & $1.49\%$\\
    PFNN & $\textbf{0.0000}$     & $0.1253$  & $0.3367$ & $0.1425$ & $\textbf{0.00\%}$ & $5.02\%$  & $8.19\%$  & $4.64\%$\\
    \midrule
    HC & $\textbf{0.0000}$ & $\textbf{0.0029}$  & $\textbf{0.0043}$ & $\textbf{0.0026}$ & $\textbf{0.00\%}$ & $\textbf{0.12\%}$  & $\textbf{0.11\%}$ & $\textbf{0.10\%}$\\
    \bottomrule
    \end{tabular}
    \end{small}
    \label{tab:high-d}
\end{table}

\subsection{High-Dimensional Heat Equation} \label{sec:exp_high}
To demonstrate that our framework can generalize to high-dimensional cases, we consider a problem of time-dependent high-dimensional PDEs defined in $\Omega\times[0,1]$ (see Appendix~\ref{sec_a8_3}), where $\Omega$ is a unit ball in $\mathbb{R}^{10}$, and $u$ is the quantity of interest.

Here, we do not consider the baselines of the xPINN, FBPINN, and PFNN-2, because the method of domain decomposition is not suitable for high-dimensional cases. And the other baselines are trained with $N_f=1000$ collocation points, $N_b=100$ boundary points, and $N_i=100$ initial points, while the HC is trained with $N_f=1000$ collocation points. All the models are trained for $\num{5000}$ Adam iterations followed by an L-BFGS optimization until convergence and tested with $N_t=\num{10000}$ points evaluated by the analytical solution to the PDEs. We refer to Table~\ref{tab:high-d} for the testing results. It can be seen that the HC achieves impressive performance in high-dimensional cases. This illustrates the effectiveness of Eq.~\eqref{eq_hc_robin} and the fact that our model can be directly applied to high-dimensional PDEs, not just geometrically complex PDEs.

\subsection{Ablation Study} \label{sec:exp_abla}
In this subsection, we are interested in the impact of the ``extra fields'' and the hyper-parameters of ``hardness'' (i.e., $\beta_s$ and $\beta_t$, where $\beta_t$ is the ``hardness'' in the temporal domain, refer to Appendix~\ref{sec_a6}).

\subsubsection{Extra fields} \label{sec:exp_abla_extra}
We test the vanilla PINN on Poisson's equation and nonlinear Schrödinger equation, where the control variable is whether to use the ``extra fields'', and the hard constraint method is not included (see Appendix~\ref{sec_a8_4}). We vary the network architecture and report the ratio of the moving variance (MovVar) of $\overline{|\nabla_{\bm{\theta}}\mathcal{L}_{\mathcal{F}}|}$ 
to that of $\overline{|{\nabla_{\tilde{\bm{\theta}}}\tilde{\mathcal{L}}_{\mathcal{F}}}|}$ (the one with the ``extra fields'') at each iteration (see Figure~\ref{fig:abla_ext}). In both PDEs, we find that models with the ``extra fields'' achieve smaller gradient oscillations during training, which is consistent with our theoretical results in Section~\ref{sec:theo}. This is because the ``extra fields'' reduce the orders of the derivatives, which in turn avoids vanishing or exploding gradients.

\begin{figure}[t]
    \centering
    \subfigure[Poisson's equation]{\label{fig:abla_poission}\includegraphics[width=.40\textwidth]{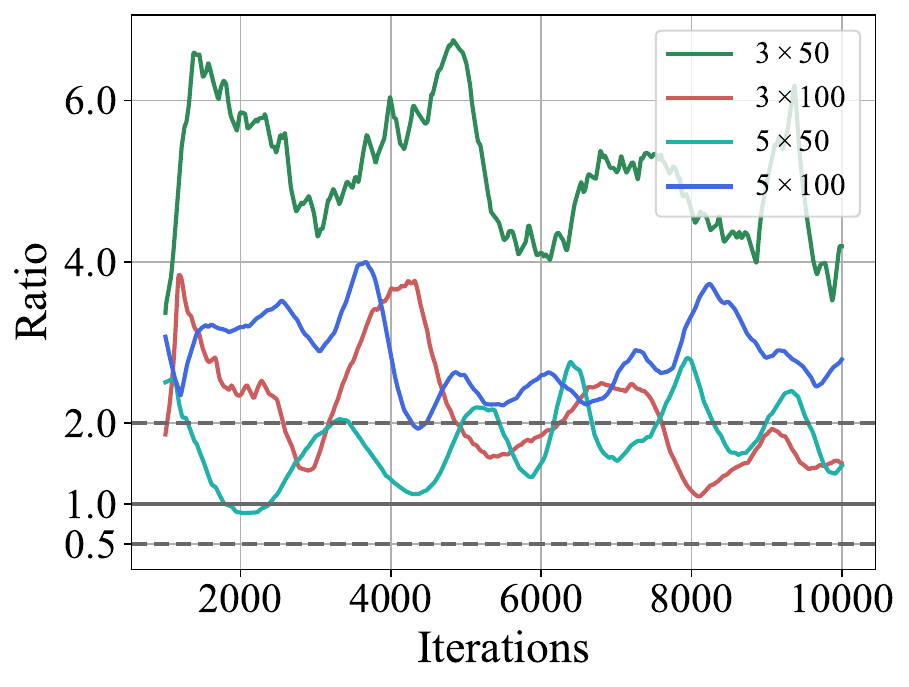}}
    \hspace{1em}
    \subfigure[Schrödinger equation]{\label{fig:abla_ns}\includegraphics[width=.40\textwidth]{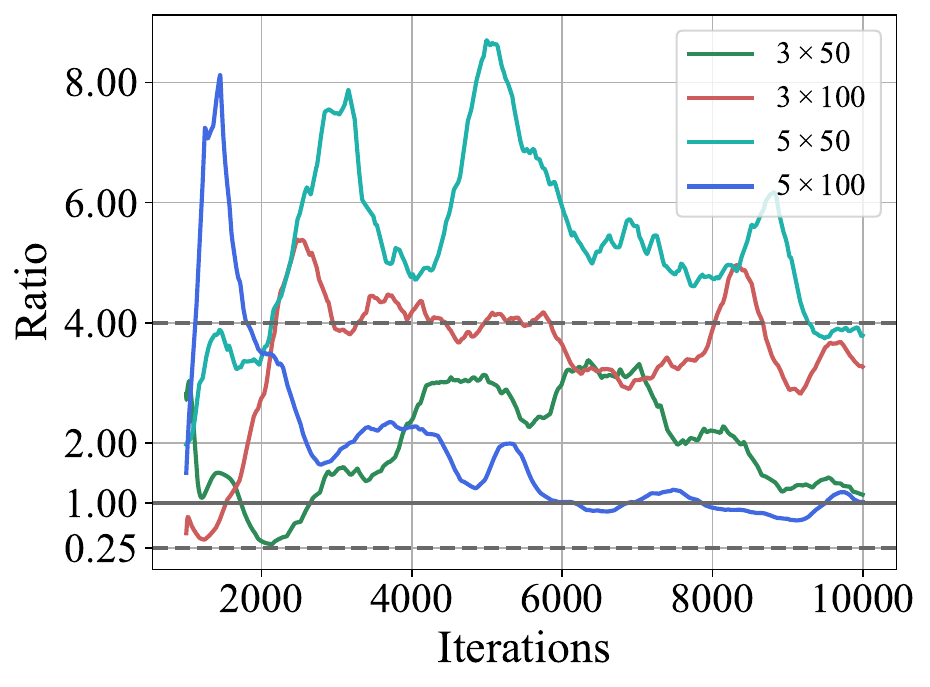}}
    \vspace{-0.1in}
    \caption{The iteration history of $\mathrm{MovVar}(\overline{|\nabla_{\bm{\theta}}\mathcal{L}_{\mathcal{F}}|})\Big/ \mathrm{MovVar}(\overline{|{\nabla_{\tilde{\bm{\theta}}}\tilde{\mathcal{L}}_{\mathcal{F}}}|})$.}
    \label{fig:abla_ext}
    \vspace{-0.2in}
\end{figure}

\begin{table}[!b]
    \caption{The MAE / MAPE of $T$ on different $\beta_s$ and $\beta_t$}
    \centering
    \begin{small}
    \begin{tabular}{lllll}
    \toprule
    & $\beta_t=1$     & $\beta_t=2$     & $\beta_t=5$ & $\beta_t=10$  \\
    \midrule
    $\beta_s=1$ & $0.3492$ / $48.21\%$     & $0.3539$ / $48.56\%$      & $0.3226$ / $44.64\%$  & $0.2889$ / $40.69\%$  \\
    $\beta_s=2$ & $0.2800$ / $40.30\%$     & $0.1670$ / $26.16\%$      & $0.2140$ / $31.72\%$  & $0.1619$ / $25.20\%$  \\
    $\beta_s=5$ & $0.1878$ / $28.68\%$     & $0.1195$ / $19.68\%$      & $0.0542$ / $10.35\%$  & $\textbf{0.0221}$ / $\textbf{5.10\%}$  \\
    $\beta_s=10$ & $0.1896$ / $29.15\%$     & $0.1104$ / $18.70\%$      & $0.0517$ / $10.56\%$  & $0.0329$ / $8.15\%$  \\
    \bottomrule
    \end{tabular}
    \end{small}
    \label{tab:abla}
\end{table}

\subsubsection{Hyper-parameters of Hardness}
We repeated the experiment in Section~\ref{sec:exp_heat} with different combinations of the values of $\beta_s$ and $\beta_t$, and all other settings stay the same. Table~\ref{tab:abla} gives the testing results (where the MAE and MAPE are the average results over $t\in[0,1]$). In general, the accuracy increases as the values of $\beta_s$ and $\beta_t$ become larger. The reason is that ``harder'' $\beta_s$ and $\beta_t$ cause the BCs to be better satisfied, which in turn makes the approximations more compliant with physical constraints. However, too large $\beta_s$ and $\beta_t$ (we emphasize that they are both on the exponential, see Eq.~\eqref{eq_unified_hc}, \eqref{eq_determine_alpha}, and \eqref{eq:hc_time_2}) can lead to numerical instability, which may be related to performance degradation.

\section{Conclusion}\label{sec:conclusion}
In this paper, we develop a unified hard-constraint framework for solving geometrically complex PDEs. With the help of ``extra fields'', we reformulate the PDEs and find the general solutions of the Dirichlet, Neumann, and Robin BCs. Based on this derivation, we propose a general formula of the hard-constraint ansatz which is applicable to time-dependent, multi-boundary, and high-dimensional cases. Besides, our theoretical analysis demonstrates that reformulation is helpful for training stability. Extensive experiments show that our method can achieve state-of-the-art performance in real-world geometrically complex as well as high-dimension PDEs, and our theoretical results are universal to general PDEs. One of the limitations is that here we only consider the three most commonly used BCs and future works may include extending this framework to more general BCs. 

\section*{Acknowledgements}
This work was supported by the National Key Research and Development Program of China (2017YFA0700904,
2020AAA0106000, 2020AAA0104304, 2020AAA0106302), NSFC Projects (Nos. 62061136001, 62076145, 62076147, U19B2034, U1811461, U19A2081, 61972224), Beijing NSF Project (No. JQ19016), BNRist (BNR2022RC01006), Tsinghua Institute for Guo Qiang, and the High Performance Computing Center, Tsinghua University. J.Z is also supported by the XPlorer Prize.



\section*{Checklist}


\begin{enumerate}

\item For all authors...
\begin{enumerate}
  \item Do the main claims made in the abstract and introduction accurately reflect the paper's contributions and scope?
    \answerYes{}
  \item Did you describe the limitations of your work?
    \answerYes{See Section~\ref{sec:conclusion}.}
  \item Did you discuss any potential negative societal impacts of your work?
    \answerYes{See Appendix~\ref{sec_a11}.}
  \item Have you read the ethics review guidelines and ensured that your paper conforms to them?
    \answerYes{}
\end{enumerate}

\item If you are including theoretical results...
\begin{enumerate}
  \item Did you state the full set of assumptions of all theoretical results?
    \answerYes{See Section~\ref{sec:theo}, \modify{Appendix~\ref{sec_a3}, \ref{sec_a4}, and \ref{sec_a7}.}}
	\item Did you include complete proofs of all theoretical results?
    \answerYes{See \modify{Appendix~\ref{sec_a3}, \ref{sec_a4}, and \ref{sec_a7}}.}
\end{enumerate}

\item If you ran experiments...
\begin{enumerate}
  \item Did you include the code, data, and instructions needed to reproduce the main experimental results (either in the supplemental material or as a URL)?
    \answerYes{\modify{See the supplementary material.}}
  \item Did you specify all the training details (e.g., data splits, hyperparameters, how they were chosen)?
    \answerYes{See Appendix~\ref{sec_a8}.}
	\item Did you report error bars (e.g., with respect to the random seed after running experiments multiple times)?
    \answerYes{\modify{See Appendix~\ref{sec_a10}.}}
	\item Did you include the total amount of compute and the type of resources used (e.g., type of GPUs, internal cluster, or cloud provider)?
    \answerYes{See Appendix~\ref{sec_a8}.}
\end{enumerate}

\item If you are using existing assets (e.g., code, data, models) or curating/releasing new assets...
\begin{enumerate}
  \item If your work uses existing assets, did you cite the creators?
    \answerYes{See Section~\ref{sec:exp_ns}.}
  \item Did you mention the license of the assets?
    \answerNA{}
  \item Did you include any new assets either in the supplemental material or as a URL?
    \answerNA{}
  \item Did you discuss whether and how consent was obtained from people whose data you're using/curating?
    \answerYes{See Section~\ref{sec:exp_ns}.}
  \item Did you discuss whether the data you are using/curating contains personally identifiable information or offensive content?
    \answerNA{No such information in the data.}
\end{enumerate}

\item If you used crowdsourcing or conducted research with human subjects...
\begin{enumerate}
  \item Did you include the full text of instructions given to participants and screenshots, if applicable?
    \answerNA{}
  \item Did you describe any potential participant risks, with links to Institutional Review Board (IRB) approvals, if applicable?
    \answerNA{}
  \item Did you include the estimated hourly wage paid to participants and the total amount spent on participant compensation?
    \answerNA{}
\end{enumerate}

\end{enumerate}


\appendix

\def\theequation{A\arabic{equation}}
\def\thefigure{A\arabic{figure}}

\section{Appendix}


\subsection{A Counter Example of a Basis of the Null Space}\label{sec_a1}

We consider a special case of Eq.~(9), where $a_i(\vx)\equiv 0$, $b_i(\vx)\equiv 1$, $g_i(\vx)\equiv 0$, and the dimension is $d=3$, where $\vn(\vx) = (n_1(\vx), n_2(\vx), n_3(\vx))$.
\begin{equation}
    n_1(\vx)p_{j1}(\vx) + n_2(\vx)p_{j2}(\vx) + n_3(\vx)p_{j3}(\vx) = 0,\qquad\vx\in\gamma_i,\qquad \forall i=1,\dots,m_j.
\end{equation}

And a counter example of $\bm{B}(\vx)$ is given by
\begin{equation}
\label{eq:couter_example}
\begin{aligned}
    \bm{B}(\vx) &= [\bm{\beta}_1(\vx), \bm{\beta}_2(\vx)],\\
    \bm{\beta}_1(\vx)  &= [n_2(\vx), -n_1(\vx), 0]^\top,\\
    \bm{\beta}_2(\vx)  &= [n_3(\vx), 0, -n_1(\vx)]^\top.
\end{aligned}
\end{equation}
One could verify that the above formula of $\bm{B}(\vx)$ is a basis of the null space, if $n_1(\vx)\neq 0, \forall \vx \in \gamma_i$. For a special case where $\gamma_i$ is a plane parallel to the x-axis, however, we have $n_1(\vx)\equiv 0, \forall \vx \in \gamma_i$. In this case, $\bm{\beta}_1(\vx), \bm{\beta}_2(\vx)$ are no longer linearly independent and cannot represent all possible solutions to $(p_{j1}(\vx),p_{j2}(\vx),p_{j3}(\vx))$. Therefore, Eq.~\eqref{eq:couter_example} is not an admissible choice for $\bm{B}(\vx)$.

\subsection{A Basis of the Null Space in Low Dimensions}\label{sec_a2}
Let $\tilde{\vn}=(a_i,b_i\vn)\big/\sqrt{a_i^2+b_i^2}$, $\tilde{g}_i=g_i\big/\sqrt{a_i^2+b_i^2}$, and $\tilde{\vp}_j=(u_j,\vp_j)$. Eq.~(9) is equivalent to
\begin{equation}\label{eq:general_bc_equiv}
    \tilde{\vn}(\vx)\cdot\tilde{\vp}_j(\vx) = \tilde{g}_i(\vx),\qquad\vx\in\gamma_i,\qquad \forall i=1,\dots,m_j.
\end{equation}

For $d=1$, we can rewrite Eq.~\eqref{eq:general_bc_equiv} as (the dimension of $\tilde{\vp}_j$ is $d+1$)
\begin{equation} 
    \tilde{n}_1(\vx)\tilde{p}_{j1}(\vx) + \tilde{n}_2(\vx)\tilde{p}_{j2}(\vx) = \tilde{g}_i(\vx),\qquad\vx\in\gamma_i,\qquad \forall i=1,\dots,m_j.
\end{equation}
And we can find that the following basis is an acceptable one
\begin{equation}\label{eq:accept_example_2}
    \bm{B}(\vx)=[\tilde{n}_2(\vx),-\tilde{n}_1(\vx)]^\top,
\end{equation}
since $\bm{B}(\vx)=0\Leftrightarrow \tilde{\vn}(\vx)=0$, and the latter contradicts the fact that $\tilde{\vn}\cdot\tilde{\vn}=1$. Then, we can use $\bm{B}$ to construct the general solution $\tilde{\vp}_j^{\gamma_i}$ under $d=1$.

And for $d=2$, Eq.~\eqref{eq:general_bc_equiv} becomes
\begin{equation}
    \tilde{n}_1(\vx)\tilde{p}_{j1}(\vx) + \tilde{n}_2(\vx)\tilde{p}_{j2}(\vx) + \tilde{n}_3(\vx)\tilde{p}_{j3}(\vx)= \tilde{g}_i(\vx),\qquad\vx\in\gamma_i,\qquad \forall i=1,\dots,m_j.
\end{equation}
An acceptable $\bm{B}(\vx)$ is given by
\begin{equation}
\label{eq:accept_example}
\begin{aligned}
    \bm{B}(\vx) &= [\bm{\beta}_1(\vx), \bm{\beta}_2(\vx), \bm{\beta}_3(\vx)],\\
    \bm{\beta}_1(\vx)  &= [0, \tilde{n}_3(\vx), -\tilde{n}_2(\vx)]^\top,\\
    \bm{\beta}_2(\vx)  &= [-\tilde{n}_3(\vx), 0, \tilde{n}_1(\vx)]^\top,\\
    \bm{\beta}_3(\vx)  &= [\tilde{n}_2(\vx), -\tilde{n}_1(\vx), 0]^\top.
\end{aligned}
\end{equation}
We note that $\bm{\beta}_1(\vx), \bm{\beta}_2(\vx), \bm{\beta}_3(\vx)$ live in the null space and $\mathrm{rank}(\bm{B}(\vx))=2$. So $\bm{B}(\vx)$ contains a basis in the null space, which can be used to construct the general solution $\tilde{\vp}_j^{\gamma_i}$ under $d=2$.

\subsection{Explanation for the General Solution}\label{sec_a3}
We first show how to find an admissible expression of $\bm{B}(\vx)$ in arbitrary dimensions with respect to Eq.~\eqref{eq:general_bc_equiv} which is equivalent to original formulation of the BC (see Eq.~(9)). We perform a Gram–Schmidt orthogonalization of $\tilde{\vn}$ (whose dimension is $d+1$) on each vector in the standard basis to get
\begin{equation}\label{eq_basis}
    \bm{\beta}_k(\vx) = \bm{e}_k - \frac{\bm{e}_k\cdot\tilde{\vn}(\vx)}{\tilde{\vn}(\vx)\cdot\tilde{\vn}(\vx)}\tilde{\vn}(\vx)=\bm{e}_k-\big(\bm{e}_k\cdot\tilde{\vn}(\vx)\big)\tilde{\vn}(\vx),\qquad k=1,\dots,d+1,
\end{equation}
where $[\bm{e}_1,\dots,\bm{e}_{d+1}]=\bm{I}_{d+1}$, and obviously all $\bm{\beta}_k(\vx),k=1,\dots,d+1,$ are in the $\mathrm{Null}(\tilde{\vn}^\top)$. We set $\bm{B}(\vx)=[\bm{\beta}_1(\vx),\dots,\bm{\beta}_{d+1}(\vx)]=\bm{I}_{d+1}-\tilde{\vn}(\vx)\tilde{\vn}(\vx)^\top$. Furthermore, we can prove that $\mathrm{rank}(\bm{B}(\vx))=d,\forall \vx \in\gamma_i$ (see Lemma~\ref{lemma_a_1}). Therefore, for $\forall\vx\in\gamma_i$, $\bm{B}(\vx)$ always contains a basis of $\mathrm{Null}(\tilde{\vn}^\top)$, and we consider such a $\bm{B}(x)$ to be an ideal choice for the general solution $\tilde{\vp}_j^{\gamma_i}$. 

\begin{lemma}\label{lemma_a_1}
$\mathrm{rank}(\bm{B}(\vx))=d$ holds for all $\vx \in\gamma_i$, where $\bm{B}(\vx)=\bm{I}_{d+1}-\tilde{\vn}(\vx)\tilde{\vn}(\vx)^\top$.
\end{lemma}
\begin{proof}
For all $\vx\in\gamma_i$, we have known that $\bm{B} = \bm{I}_{d+1}-\tilde{\vn}\tilde{\vn}^\top$, where $\tilde{\vn}\cdot\tilde{\vn}=\tilde{\vn}^\top\tilde{\vn}=1$, and
\begin{equation}
    \bm{B}\tilde{\vn}=(\bm{I}_{d+1}-\tilde{\vn}\tilde{\vn}^\top)\tilde{\vn}=\tilde{\vn}-\tilde{\vn}\tilde{\vn}^\top\tilde{\vn}=\tilde{\vn}-\tilde{\vn}=\bm{0}.
\end{equation}
Hence, $\mathrm{rank}(\bm{B})\leq d$. Besides, we notice that $\bm{H}=\bm{I}_{d+1}-2\tilde{\vn}\tilde{\vn}^\top$ is a Householder matrix, which is an invertible matrix, since
\begin{equation}
    \bm{H}^\top \bm{H}=(\bm{I}_{d+1}-2\tilde{\vn}\tilde{\vn}^\top)^2=\bm{I}_{d+1}-4\tilde{\vn}\tilde{\vn}^\top+4\tilde{\vn}\tilde{\vn}^\top\tilde{\vn}\tilde{\vn}^\top=\bm{I}_{d+1}-4\tilde{\vn}\tilde{\vn}^\top+4\tilde{\vn}\tilde{\vn}^\top=\bm{I}_{d+1}.
\end{equation}
So $\mathrm{rank}(\bm{H})=d+1$, and we have
\begin{equation}
    d+1=\mathrm{rank}(\bm{I}_{d+1}-2\tilde{\vn}\tilde{\vn}^\top)\leq \mathrm{rank}(\bm{I}_{d+1}-\tilde{\vn}\tilde{\vn}^\top) + \mathrm{rank}(\tilde{\vn}\tilde{\vn}^\top)=\mathrm{rank}(\bm{B})+1,
\end{equation}
which can deduce $d\leq\mathrm{rank}(\bm{B})$. Therefore, $\mathrm{rank}(\bm{B})=d$. 
\end{proof}

Finally, we show that the general solution in Eq.~(10) satisfies the BC in Eq.~\eqref{eq:general_bc_equiv}.
\begin{equation}
\begin{aligned}
    \tilde{\vn}(\vx) \cdot \tilde{\vp}_j^{\gamma_i}(\vx) &= \tilde{\vn}(\vx) \cdot \bm{B}(\vx) \mathrm{NN}_j^{\gamma_i}(\vx) + \tilde{\vn}(\vx) \cdot \tilde{\vn}(\vx) \tilde{g}_i(\vx)\\
    &= \tilde{\vn}(\vx) \cdot \left( \bm{I}_{d+1}-\tilde{\vn}(\vx)\tilde{\vn}(\vx)^\top \right) \mathrm{NN}_j^{\gamma_i}(\vx) + \tilde{g}_i(\vx)\\
    &= \tilde{g}_i(\vx),
\end{aligned}
\end{equation}
where we omit the trainable parameters for simplicity. Besides, the discussion of $\bm{B}(\vx)$ in low-dimensional cases (i.e., $d=1$ and $d=2$, see Appendix~\ref{sec_a3}) is similar, and we will leave it to the reader.

\subsection{Theoretical Guarantee of the Constructed Ansatz}\label{sec_a4}

In Appendix~\ref{sec_a3}, we have demonstrated that $\bm{B}(\vx)$ contains a basis of the null space of the BC for $\forall\vx \in \gamma_i$ and the general solution in Eq.~(10) satisfies the corresponding BC. In this subsection, we will show that our constructed ansatz in Eq.~(11) is theoretically correct. We first prove that the ansatz in Eq.~(11) satisfies all the BCs under the following assumptions.
\begin{assumption}\label{ass_a_0}
The problem domain $\Omega$ is bounded.
\end{assumption}
\begin{assumption}\label{ass_a_1}
The shortest distance between $\gamma_1,\dots,\gamma_{m_j}$ is greater than zero for $j=1,\dots,n$.
\end{assumption}
\begin{assumption}\label{ass_a_2}
All the extended distance functions $l^{\partial\Omega},l^{\gamma_i},i=1,\dots,m_j$ are continuous and satisfy that $\min_{\vx \in \partial\Omega \setminus \gamma_i} l^{\gamma_i}(\vx) \geq C_i$, $\forall \vx\in\gamma_i$, $i=1,\dots,m_j$ for $j=1,\dots,n$, where $C_i$ is a positive constant.
\end{assumption}

\begin{theorem}
$\forall \epsilon > 0$, there exists $\beta^0_s\in \mathbb{R}$, such that
\begin{equation}
    \left| \tilde{\vn} (\vx) \cdot (\hat{u}_j, \hat{\vp}_j) - \tilde{g}_i(\vx) \right| < \epsilon,
\end{equation}
holds for all $\beta_s > \beta^0_s$, $\vx\in \gamma_i$, $i=1,\dots,m_j$, $j=1,\dots,n$, where $\tilde{\vn}=(a_i,b_i\vn)\big/\sqrt{a_i^2+b_i^2}$, $\tilde{g}_i=g_i\big/\sqrt{a_i^2+b_i^2}$, and $\bm{B}(\vx)=\bm{I}_{d+1}-\tilde{\vn}(\vx)\tilde{\vn}(\vx)^\top$.
\end{theorem}

\begin{proof}
For any $\vx\in \gamma_i$, we have $l^{\partial \Omega}(\vx) = 0$ according to the definition of the extended distance functions. Thus, $(\hat{u}_j, \hat{\vp}_j)$ is now equal to
\begin{equation}
    (\hat{u}_j,\hat{\vp}_j) = \sum_{k=1}^{m_j} \exp{\big[-\alpha_k l^{\gamma_k}(\vx)\big]}\tilde{\vp}_j^{\gamma_k}(\vx),
\end{equation}
where we omit the trainable parameters. Then, according to Assumptions~\ref{ass_a_0} $\sim$ \ref{ass_a_2}, we can choose a sufficiently large $\beta^i_s$ (see Eq.~(12) for the relationship between $\alpha_i$ and $\beta_s$), such that
\begin{equation}
\begin{aligned}
    \left| \tilde{\vn} (\vx) \cdot (\hat{u}_j, \hat{\vp}_j) - \tilde{g}_i(\vx) \right| &= \left| \tilde{\vn} (\vx) \cdot \left( \sum_{k=1}^{m_j} \exp{\big[-\alpha_k l^{\gamma_k}(\vx)\big]}\tilde{\vp}_j^{\gamma_k}(\vx) \right) - \tilde{g}_i(\vx) \right|\\
    &< \left| \tilde{\vn} (\vx) \cdot \exp{\big[-\alpha_i l^{\gamma_i}(\vx)\big]}\tilde{\vp}_j^{\gamma_i}(\vx) - \tilde{g}_i(\vx) \right| \\
    &\qquad+ \left| \tilde{\vn} (\vx) \cdot \sum_{k\neq i} \exp{\big[-\alpha_k l^{\gamma_k}(\vx)\big]}\tilde{\vp}_j^{\gamma_k}(\vx) \right|\\
    &\leq \left| \tilde{\vn} (\vx) \cdot \tilde{\vp}_j^{\gamma_i}(\vx) - \tilde{g}_i(\vx) \right| + \left| \sum_{k\neq i} \exp{\big[-\alpha_k l^{\gamma_k}(\vx)\big]}\tilde{\vp}_j^{\gamma_k}(\vx) \right|\\
    &= \left| \sum_{k\neq i} \exp{\big[-\alpha_k l^{\gamma_k}(\vx)\big]}\tilde{\vp}_j^{\gamma_k}(\vx) \right|\\
    &< \epsilon,
\end{aligned}
\end{equation}
where we note that $l^{\gamma_i}(\vx)=0$ and $l^{\gamma_k}(\vx)>0,\forall k\neq i$ for all $\vx\in\gamma_i$. Let $\beta^0_s = \max \{\beta^i_s \mid i=1,\dots,m_j\}$, then according to the arbitrariness of $j$, we conclude that the theorem holds.
\end{proof}

Next, we will prove that our ansatz can approximate the solution to the PDEs with arbitrarily low errors under the following assumptions in addition to Assumptions~\ref{ass_a_0} $\sim$ \ref{ass_a_2}.

\begin{assumption}\label{ass_a_3}
The solution to the PDEs $\vu(\vx)$ is unique, bounded, and at least first-order continuous by element.
\end{assumption}

\begin{assumption}\label{ass_a_4}
$a_i(\vx)$, $b_i(\vx)$, and $g_i(\vx)$ are continuous (hence $\tilde{g}_i(\vx)$ is continuous, too) in $\gamma_i$ for $i=1,\dots,m_j$, $j=1,\dots,n$.
\end{assumption}

\begin{assumption}\label{ass_a_5}
Since $\bm{B}(\vx)=\bm{I}_{d+1}-\tilde{\vn}(\vx)\tilde{\vn}(\vx)^\top$ is a real symmetric matrix, we can perform an orthogonal diagonalization $\bm{B}(\vx) = \bm{P}(\vx)^\top \Lambda(\vx) \bm{P}(\vx)$, where $\Lambda(\vx)=\mathrm{diag}(\lambda_1(\vx),\dots,\lambda_d(\vx), 0)$, $\lambda_1(\vx)>\cdots>\lambda_d(\vx)>0$. We assume that $\tilde{\vn}(\vx)$, $\bm{P}(\vx)$, and $\Lambda(\vx)$ are piece-wise continuous by element in $\gamma_i$ for $i=1,\dots,m_j$, $j=1,\dots,n$.
\end{assumption}

To begin with, we prove this lemma.
\begin{lemma}\label{lemma_a_6}
$\forall \epsilon > 0$, there exists $\bm{\theta}^{\gamma_i}_j\in {\Theta}^{\gamma_i}_j$, such that
\begin{equation}
    \left \|  \tilde{\vp}^{\gamma_i}_j(\vx; \bm{\theta}^{\gamma_i}_j)-\bm{q}(\vx) \right\|_1 < \epsilon,
\end{equation}
holds for all $\vx\in \gamma_i$ if $\bm{q}(\vx)$ is continuous in $\gamma_i$ and satisfies the BC (i.e., $\tilde{\vn}(\vx)\cdot \bm{q}(\vx)=\tilde{g}_i(\vx),\forall \vx \in \gamma_i$), where ${\Theta}^{\gamma_i}_j$ is the parameter space of the neural network ${\mathrm{NN}}^{\gamma_i}_j$, $\| \cdot \|_1$ is the 1-norm of matrices (operator norm), and $\tilde{\vp}^{\gamma_i}_j$ as well as $\bm{q}$ are both of dimension $d+1$. The above conclusion holds for all $i=1,\dots,m_j$, $j=1,\dots,n$.
\end{lemma}
\begin{proof}
From Eq.~\eqref{eq_basis} and Lemma~\ref{lemma_a_1}, we know that $\bm{B}(\vx)$ contain a basis of $\mathrm{Null}(\tilde{\vn}^\top)$. Since $\bm{q}(\vx)$ satifies the BC, we can represent it as
\begin{equation}\label{eq_represent_q}
    \bm{q}(\vx) = \bm{B}(\vx) \bm{r}(\vx) + \tilde{\vn}(\vx) \tilde{g}_i(\vx).
\end{equation}
Then we will show that there exists a piece-wise continuous choice of $\bm{r}(\vx)$. We rewrite Eq.~\eqref{eq_represent_q} as
\begin{equation}\label{eq_represent_q_2}
    \bm{q}(\vx) =  \bm{P}(\vx)^\top \Lambda(\vx) \bm{P}(\vx) \bm{r}(\vx) + \tilde{\vn}(\vx) \tilde{g}_i(\vx).
\end{equation}
Since $\bm{B}(\vx)$ has $d+1$ column vectors, which is greater than the dimension of $\mathrm{Null}(\tilde{\vn}^\top)$ (i.e., $d$), the choice of $\bm{q}(\vx)$ is not unique. We can choose a particular $\bm{q}(\vx)$ which satisfies that the last element of $\bm{P}(\vx)\bm{r}(\vx)$ is zero (i.e., $\bm{P}(\vx)\bm{r}(\vx)=[\dots,0]^\top$). Next, we continue with the equivalent transformation of Eq.~\eqref{eq_represent_q_2}.
\begin{equation}
\begin{aligned}
    &&  \bm{P}(\vx)^\top \Lambda(\vx) \bm{P}(\vx) \bm{r}(\vx) + \tilde{\vn}(\vx) \tilde{g}_i(\vx)&=\bm{q}(\vx) ,\\
    \Longleftrightarrow&& \quad   \bm{P}(\vx)^\top \Lambda(\vx) \bm{P}(\vx) \bm{r}(\vx) &= \bm{q}(\vx)- \tilde{\vn}(\vx) \tilde{g}_i(\vx),\\
    \Longleftrightarrow&& \quad   \Lambda(\vx) \bm{P}(\vx) \bm{r}(\vx) &= \bm{P}(\vx)\left(\bm{q}(\vx)- \tilde{\vn}(\vx) \tilde{g}_i(\vx)\right) ,\\
    \Longleftrightarrow&& \quad  \mathrm{diag}(1,\dots,1,0)  \bm{P}(\vx) \bm{r}(\vx) &= \Lambda'(\vx) \bm{P}(\vx)\left(\bm{q}(\vx)- \tilde{\vn}(\vx) \tilde{g}_i(\vx)\right),\\
    \Longleftrightarrow&& \quad   \bm{P}(\vx) \bm{r}(\vx) &= \Lambda'(\vx) \bm{P}(\vx)\left(\bm{q}(\vx)- \tilde{\vn}(\vx) \tilde{g}_i(\vx)\right),
\end{aligned}
\end{equation}
where $\Lambda'(\vx)=\mathrm{diag}(1/\lambda_1(\vx),\dots,1/\lambda_d(\vx), 0)$. The last equivalence holds because the last element of $\bm{P}(\vx) \bm{r}(\vx)$ is always zero. From Assumption~\ref{ass_a_4} and \ref{ass_a_5}, combining the above formula, we have that $\bm{P}(\vx) \bm{r}(\vx)$ is piece-wise continuous by element. Noticing that $\bm{r}(\vx) = \bm{P}(\vx)^\top \bm{P}(\vx) \bm{r}(\vx)$, we know that the $\bm{r}(\vx)$ we chosen is also piece-wise continuous by element.

We notice that
\begin{equation}
\begin{aligned}
    \left \|  \bm{B}(\vx) \right\|_1 &= \left \| \bm{I}_{d+1}-\tilde{\vn}(\vx)\tilde{\vn}(\vx)^\top  \right\|_1\\
    &\le \left \| \bm{I}_{d+1}  \right\|_1 + \left \| \tilde{\vn}(\vx)\tilde{\vn}(\vx)^\top  \right\|_1\\
    &\le 1 + d + 1\\
    &= d + 2.
\end{aligned}
\end{equation}
According to the Universal Approximation of neural networks \cite{llanas2008constructive}, $\forall\epsilon>0$, there exists $\bm{\theta}^{\gamma_i}_j\in {\Theta}^{\gamma_i}_j$, such that
\begin{equation}
    \left \| \mathrm{NN}^{\gamma_i}_j(\bm{x};\bm{\theta}^{\gamma_i}_j) - \bm{r}(\vx) \right\|_1 < \frac{\epsilon}{d+2},
\end{equation}
holds for all $\vx\in \gamma_i$. Therefore,
\begin{equation}
\begin{aligned}
    \left \|  \tilde{\vp}^{\gamma_i}_j(\vx; \bm{\theta}^{\gamma_i}_j)-\bm{q}(\vx) \right\|_1 &=  \left \| \bm{B}(\vx)\left(\mathrm{NN}^{\gamma_i}_j(\bm{x};\bm{\theta}^{\gamma_i}_j) - \bm{r}(\vx) \right)\right\|_1\\
    &\le \left \|  \bm{B}(\vx) \right\|_1  \left \| \mathrm{NN}^{\gamma_i}_j(\bm{x};\bm{\theta}^{\gamma_i}_j) - \bm{r}(\vx) \right\|_1\\
    &< \epsilon.
\end{aligned}
\end{equation}
According to the arbitrariness of $i$ and $j$, we conclude that the lemma holds.
\end{proof}

Finally, we state the following theorem.
\begin{theorem}
$\forall \epsilon > 0$, there exists $\beta_s\in\mathbb{R}$, $\bm{\theta}^{\mathrm{main}}_j\in\Theta^{\mathrm{main}}_j$, $\bm{\theta}^{\gamma_i}_j\in {\Theta}^{\gamma_i}_j,i=1,\dots,m_j$, such that
\begin{equation}\label{eq_final_theo}
    \left \| (\hat{u}_j,\hat{\vp}_j)- (u_j, \nabla u_j) \right\|_1 < \epsilon,
\end{equation}
holds for all $\vx\in \Omega\cup\partial\Omega$, $j=1,\dots,n$, where ${\Theta}_{*}$ is the parameter space of the corresponding neural network, $\| \cdot \|_1$ is the 1-norm. The ground truth solution to the PDEs is $\vu(\vx)=(u_1(\vx),\dots,u_n(\vx))$.
\end{theorem}
\begin{proof}
For $\vx\in\gamma_i$, $i=1,\dots,m_j$, $(u_j, \nabla u_j)$ is continuous (according to Assumption~\ref{ass_a_3}) and satisfies the BC (which the solution needs to meet). From Lemma~\ref{lemma_a_6}, the definition of $(\hat{u}_j,\hat{\vp}_j)$ in Eq.~(11), and Assumptions~\ref{ass_a_0} $\sim$ \ref{ass_a_2}, we can find $\bm{\theta}^{\gamma_i}_j\in {\Theta}^{\gamma_i}_j$ and a large enough $\beta_s^i$ such that Eq.~\eqref{eq_final_theo} holds for all $\vx\in\gamma_i$.

Then we fix $\beta_s = \max \{\beta^i_s \mid i=1,\dots,m_j\}$ and $\bm{\theta}^{\gamma_i}_j\in {\Theta}^{\gamma_i}_j, i=1,\dots,m_j$ (which are what we determined for $\vx\in\gamma_i,i=1,\dots,m_j$). From Assumption~\ref{ass_a_0}, we have $|l^{\partial \Omega}(\vx)| < C,\forall \vx\in\Omega$, where $C$ is a positive constant. For $\vx\in\Omega$, according to the Universal Approximation Theorem of neural networks \cite{chen1995universal}, there exists $\bm{\theta}^{\mathrm{main}}_j \in \Theta^{\mathrm{main}}_j$ satisfying that
\begin{equation}
    \left \| \frac{(u_j, \nabla u_j) - \sum_{i=1}^{m_j} \exp{\big[-\alpha_i l^{\gamma_i}(\vx)\big]}\tilde{\vp}_j^{\gamma_i}(\vx;\bm{\theta}_j^{\gamma_i})}{l^{\partial\Omega}(\vx)} - \mathrm{NN}^{\mathrm{main}}_j(\vx;\bm{\theta}^{\mathrm{main}}_j)\right\|_1 < \frac{\epsilon}{C}.
\end{equation}
Therefore,
\begin{equation}
\begin{aligned}
    &\left \| (\hat{u}_j,\hat{\vp}_j)- (u_j, \nabla u_j) \right\|_1\\
    =& \bigg \| (u_j, \nabla u_j) - \sum_{i=1}^{m_j} \exp{\big[-\alpha_i l^{\gamma_i}(\vx)\big]}\tilde{\vp}_j^{\gamma_i}(\vx;\bm{\theta}_j^{\gamma_i}) 
    - l^{\partial\Omega}(\vx)\mathrm{NN}^{\mathrm{main}}_j(\vx;\bm{\theta}^{\mathrm{main}}_j) \bigg\|_1\\
    =& \bigg \| \frac{(u_j, \nabla u_j) - \sum_{i=1}^{m_j} \exp{\big[-\alpha_i l^{\gamma_i}(\vx)\big]}\tilde{\vp}_j^{\gamma_i}(\vx;\bm{\theta}_j^{\gamma_i})}{l^{\partial\Omega}(\vx)} 
    - \mathrm{NN}^{\mathrm{main}}_j(\vx;\bm{\theta}^{\mathrm{main}}_j)\bigg\|_1 \left| l^{\partial \Omega}(\vx) \right|\\
    <& \frac{\epsilon}{C}\cdot C = \epsilon.
\end{aligned}
\end{equation}
According to the arbitrariness of $j$, we have proven this theorem.
\end{proof}

Besides, we note that it is easy to extend the above theorems to time-dependent cases (the ansatz is given in Appendix~\ref{sec_a6}), which will not be discussed separately here.

\subsection{Extension of the Parameter Functions in the BCs}\label{sec_a5}
In Eq.~(9), it is noted that $a_i$, $b_i$, $\vn$ or $g_i$ may be only defined at $\gamma_i$. But they are included in our ansatz (see Eq.~(10) and Eq.~(11)), which is defined in $\Omega\cup\partial\Omega$. So we need to extend their definition smoothly to $\Omega\cup\partial\Omega$, using interpolation or approximation via neural networks. We consider the airfoil boundary (i.e, $\gamma_{\mathrm{airfoil}}$) in Section~5.3 as a motivating example. 

\begin{figure}[t!]
    \centering
    \includegraphics[width=0.8\linewidth]{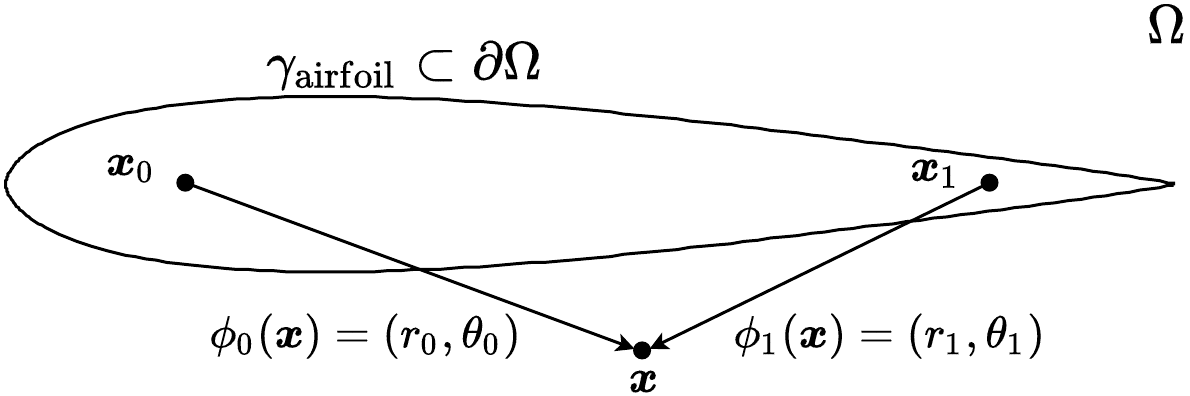}
    \vspace{-0.1in}
    \caption{Illustration of the extension from $\gamma_{\mathrm{airfoil}}$ to $\Omega\cup\partial\Omega$.}
    \label{fig:airfoil_extension}
    \vspace{-0.1in}
\end{figure}

Supposing $f(\vx)$ is only defined in $\gamma_{\mathrm{airfoil}}$, our task is to extend its definition to $\Omega\cup\partial\Omega$. As shown in Figure~\ref{fig:airfoil_extension}, we first place two reference points (i.e., $\vx_0$ and $\vx_1$) on the front and rear half of the airfoil. For any $\vx\in\Omega\cup\partial\Omega$, it can be expressed as polar coordinates with respect to $\vx_0$ and $\vx_1$, respectively. We concatenate the two polar coordinates to form a new space. We next perform interpolation and approximation under the new space. This is because in the new space, we can better characterize the shape of the airfoil. It is true that there are many ways for coordinate transformations, not limited to the example here. 

As for the interpolation, we can sample several points at the $\gamma_{\mathrm{airfoil}}$ to obtain the dataset $\{ ((\theta_0^{(i)}, \theta_1^{(i)}), f^{(i)}) \}_{i=1}^N$. For any $\vx\in\Omega\cup\partial\Omega$, we generate the corresponding extended $f(\vx)$ by interpolating in the dataset. The interpolation method used here depends on the smoothness requirements of the ansatz. In addition, the number of reference points can also be changed, and in experiments, we found that only one reference point is enough.

Approximation via neural networks is a general method that does not require manual design. In this case, we can sample several points at the $\gamma_{\mathrm{airfoil}}$ to construct our dataset $\{ ((\phi_0(\vx^{(i)}), \phi_1(\vx^{(i)})), f^{(i)}) \}_{i=1}^N$, followed by training a neural network on the dataset, i.e. $\mathrm{NN}(\phi_0(\vx^{(i)}), \phi_1(\vx^{(i)}))\approx f^{(i)}$. For any $\vx\in\Omega\cup\partial\Omega$, we take $\mathrm{NN}(\phi_0(\vx), \phi_1(\vx))$ as the corresponding extended $f(\vx)$. We can also train the neural network in the original space. However, experimental results show that training in the new space can achieve better results. The reason may be that the complex geometry become smoother and easier to learn in the new space.

It is worth noting that, in addition to the cases mentioned above, the extended distance functions $l(\vx)$ (here we omit the superscript and see Eq.~(4) for its definition) may also need to be handled similarly. Because for the complex geometry, the distance function can be very complex and we may want to replace it with a cheap surrogate model. The methods are similar, including approximating the distance function with a neural network, or constructing splines function \citep{sheng2021pfnn}.

\subsection{The Hard-Constraint Framework for Time-dependent PDEs}\label{sec_a6}
In this section, we consider the following time-dependent PDEs
\begin{equation}
    \mathcal{F}[\vu(\vx, t)]=\bm{0},\qquad\vx=(x_1,\dots,x_d)\in\Omega,t\in(0,T],
\end{equation}
where $t$ is the temporal coordinate, and the other notations are the same as those in Section~3.1. For each $u_j, j=1,\dots,n$, we pose suitable boundary conditions (BCs)
\begin{equation}
    a_i(\vx, t)u_j+b_i(\vx,t)\big( \vn(\vx)\cdot\nabla u_j \big) = g_i(\vx,t),\quad\vx\in\gamma_i,t\in(0,T],\quad \forall i=1,\dots,m_j,
\end{equation}
and an initial condition (IC)
\begin{equation}
    u_j(\vx, 0) = f_j(\vx),\quad\vx\in\Omega.
\end{equation}

Following the pipeline described in Section~3.2, we can find the general solution $\tilde{\vp}_j^{\gamma_i}(\vx,t)$ as
\begin{equation}
    \tilde{\vp}_j^{\gamma_i}=\bm{B}(\vx, t)\mathrm{NN}_j^{\gamma_i}(\vx, t) + \tilde{\vn}(\vx,t)\tilde{g}_i(\vx,t),
\end{equation}
where $\tilde{\vn}=(a_i,b_i\vn)\big/\sqrt{a_i^2+b_i^2}$, $\tilde{g}_i=g_i\big/\sqrt{a_i^2+b_i^2}$, $\mathrm{NN}_j^{\gamma_i}:\mathbb{R}^{d+1}\to\mathbb{R}^{d+1}$ is a neural network, and $\bm{B}(\vx, t)=\bm{I}_{d+1}-\tilde{\vn}(\vx, t){\tilde{\vn}(\vx,t)}^\top$. And we omit the trainable parameters of neural networks for neatness. 

Finally, we can construct our ansatz $(\hat{u}_j,\hat{\vp}_j)$ as
\begin{subequations}\label{eq_time_ansatz}
\begin{align}
    ({u}_j^\dagger,\hat{\vp}_j) &= l^{\partial \Omega}(\vx) \mathrm{NN}^{\mathrm{main}}_j(\vx,t) + \sum_{i=1}^{m_j} \exp{\big[-\alpha_i l^{\gamma_i}(\vx)\big]}\tilde{\vp}_j^{\gamma_i}(\vx,t),&& \forall j=1,\dots,n\label{eq:hc_time},\\
    \hat{u}_j &= {u}_j^\dagger(\vx, t) \big(1-\exp{[-\beta_t t]}\big) + f_j(\vx) \exp{[-\beta_t t]},&& \forall j=1,\dots,n,\label{eq:hc_time_2}
\end{align}
\end{subequations}
where ${u}_j^\dagger$ is an intermediate variable that incorporates hard constraints in spatial dimensions, $\mathrm{NN}^{\mathrm{main}}_j:\mathbb{R}^{d+1}\to\mathbb{R}^{d+1}$ is the main neural network, $l^{\partial \Omega}, l^{\gamma_i},i=1,\dots,m_j$ are extended distance functions (see Eq.~(4)), $\alpha_i~(i=1,\dots,m_j)$ is determined in Eq.~(12), and $\beta_t\in\mathbb{R}$ is a hyper-parameter of the ``hardness'' in the temporal domain. 

\subsection{Supplements to the Theoretical Analysis}\label{sec_a7}
In this section, we first give some supplements to the problem setting in Section~4. Then we present the proof of Theorem~4.1. Finally, we will characterize the mechanism described in Section~4 with another tool, the condition number.

\subsubsection{Supplements to the Problem Setting}\label{sec_a7_1}
As mentioned in Section~4, we consider the following 1D Poisson's equation
\begin{subequations}
\label{eq:theo_possion}
\begin{align}
    \Delta u(x)&=-a^2\sin{ax},&&x\in(0,2\pi),\label{eq:theo_possion_1}\\
    u(x)&=0,&&x=0\lor x=2\pi,
\end{align}
\end{subequations}
where $a\in\mathbb{R}$ and $u$ is the physical quantity of interest. 
Here we use a single-layer neural network of width $K$ as our ansatz, i.e., $\hat{u}=\vc^\top \sigma(\vw x+\vb)$, where $\vc,\vw,\vb\in\mathbb{R}^K$, $\sigma$ is an element-wise activation function (for simplicity, we take $\sigma$ as $\tanh$). To study the impact of the \textit{extra fields} alone, we train $\hat{u}$ in a soft-constrained manner. For ease of discussion, we consider the loss function in continuous form
\begin{equation}
    \mathcal{L}(\bm{\theta})=\mathcal{L}_{\mathcal{F}}(\bm{\theta}) + \mathcal{L}_{\mathcal{B}}(\bm{\theta})\approx\frac{1}{2\pi}\int_0^{2\pi}\big(\Delta \hat{u}(x) + a^2\sin(ax)\big)^2 \diff{x} + \big(\hat{u}(0)\big)^2 + \big(\hat{u}(2\pi)\big)^2,
\end{equation}
where $\bm{\theta}=(\vc,\vw,\vb)$ is a set of trainable parameters. 

Let $p=\nabla u = \diff{u}/\diff{x}$. We reformulate Eq.~\eqref{eq:theo_possion} via the \textit{extra fields} to obtain
\begin{subequations}
\label{eq:theo_possion_after_transformation}
\begin{align}
    \nabla p(x)&=-a^2\sin{ax},&&x\in(0,2\pi),\\
    p(x)&=\nabla u(x),&&x\in(0,2\pi),\\
    u(x)&=0,&&x=0\lor x=2\pi.
\end{align}
\end{subequations}
Our ansatz becomes $\hat{u}=\vc^\top \sigma(\vw x+\vb)$ and $\hat{p}=\vc_{p}^\top \sigma(\vw x+\vb)$, where $\vc_{p}\in\mathbb{R}^K$ is a weight vector with respect to the output $\hat{p}$. We can see that the loss term of the BC does not change while that of the PDE becomes
\begin{equation}
    \tilde{\mathcal{L}}_{\mathcal{F}}(\tilde{\bm{\theta}})\approx\frac{1}{2\pi}\int_0^{2\pi}\Big[\big(\nabla \hat{p}(x) + a^2\sin(ax)\big)^2 + \big(\hat{p}(x)-\nabla \hat{u}(x)\big)^2\Big] \diff{x},
\end{equation}
where $\tilde{\bm{\theta}}=(\vc,\vw,\vb,\vc_{p})$ is a set of trainable parameters.

\subsubsection{Proof of Theorem~4.1}\label{sec_a7_2}
In this part, we provide detailed proof of Theorem 4.1.

We first derive the derivatives of the ansatz for the original PDEs (we recall that $\sigma$ is $\tanh$ and we have $\sigma' = 1 - \sigma^2$) 
\begin{subequations}
\begin{align}
    \frac{\diff{\hat{u}}}{\diff{x}} &= \vc^\top \Big[\big(\bm{1}-\sigma^2(\vw x+\vb)\big)\circ\vw\Big],\\
    \frac{\Diff2{\hat{u}}}{\diff{x}^2} &= -2\vc^\top \Big[\sigma(\vw x+\vb)\circ\big(\bm{1}-\sigma^2(\vw x+\vb)\big)\circ\vw^2\Big].
\end{align}
\end{subequations}
We will abbreviate $\sigma(\vw x+\vb)$ as $\vg$, and then we have
\begin{subequations}
\begin{align}
    \frac{\diff{\hat{u}}}{\diff{x}} &= \vc^\top \big[(\bm{1}-\vg^2)\circ\vw\big],\\
    \frac{\Diff2{\hat{u}}}{\diff{x}^2} &= -2\vc^\top \big[\vg\circ(\bm{1}-\vg^2)\circ\vw^2\big].
\end{align}
\end{subequations}
Now, we can provide a bound for $(\partial \mathcal{L}_{\mathcal{F}} / \partial \vc)^\top$ as
\begin{equation}
\begin{aligned}
\bigg\lvert \Big(\frac{\partial \mathcal{L}_{\mathcal{F}}}{\partial \vc}\Big)^\top \bigg\rvert &= \frac{1}{2\pi}\Bigg\lvert \int_0^{2\pi}2\Big(\frac{\Diff2{\hat{u}}}{\diff{x}^2} + a^2\sin(ax)\Big) \cdot \bigg(\partial \Big( \frac{\Diff2{\hat{u}}}{\diff{x}^2} \Big) \Big/ \partial \vc \bigg) \diff{x}\Bigg\rvert\\
&= \frac{2}{\pi}\Bigg\lvert\int_0^{2\pi}\Big(-2\vc^\top \big[\vg\circ(\bm{1}-\vg^2)\circ\vw^2\big] + a^2\sin(ax)\Big)\\
&\qquad\cdot \big[\vg\circ(\bm{1}-\vg^2)\circ\vw^2\big] \diff{x}\Bigg\rvert\\
&\le \frac{2}{\pi} \Bigg(\int_0^{2\pi}\Big(-2\vc^\top \big[\vg\circ(\bm{1}-\vg^2)\circ\vw^2\big] + a^2\sin(ax)\Big)^2\diff{x} \Bigg)^{\frac{1}{2}}\\
&\qquad\cdot \Bigg(\int_0^{2\pi}\big[\vg\circ(\bm{1}-\vg^2)\circ\vw^2\big]^2\diff{x} \Bigg)^{\frac{1}{2}}\qquad (\mathrm{Cauchy-Schwarz})\\
&\le \frac{2}{\pi} \Bigg(\int_0^{2\pi}\big(2|\vc|^\top\vw^2 + a^2\big)^2\diff{x} \Bigg)^{\frac{1}{2}}\cdot \Bigg(\int_0^{2\pi}\vw^4\diff{x} \Bigg)^{\frac{1}{2}}\\
&= 4 \big(2|\vc|^\top\vw^2 + a^2\big) \vw^2\\
&\le 8 \big(|\vc|^\top\vw^2 + a^2\big) \vw^2,
\end{aligned}
\end{equation}
where $\le$ between two vectors is an element-wise comparison. Thus, $(\partial \mathcal{L}_{\mathcal{F}} / \partial \vc)^\top$ can be bounded by 
\begin{equation} \label{eq:bound_c}
    \bigg\lvert \Big(\frac{\partial \mathcal{L}_{\mathcal{F}}}{\partial \vc}\Big)^\top \bigg\rvert =  \mathcal{O}\big(|\vc|^\top\vw^2+a^2\big)\cdot\vw^2.
\end{equation}
Similarly, for $(\partial \mathcal{L}_{\mathcal{F}} / \partial \vw)^\top$, we have
\begin{equation}
\begin{aligned}
\bigg\lvert \Big(\frac{\partial \mathcal{L}_{\mathcal{F}}}{\partial \vw}\Big)^\top \bigg\rvert &= \frac{1}{2\pi}\Bigg\lvert \int_0^{2\pi}2\Big(\frac{\Diff2{\hat{u}}}{\diff{x}^2} + a^2\sin(ax)\Big) \cdot \bigg(\partial \Big( \frac{\Diff2{\hat{u}}}{\diff{x}^2} \Big) \Big/ \partial \vw \bigg) \diff{x}\Bigg\rvert\\
&= \frac{2}{\pi}\Bigg\lvert\int_0^{2\pi}\Big(-2\vc^\top \big[\vg\circ(\bm{1}-\vg^2)\circ\vw^2\big] + a^2\sin(ax)\Big)\\ 
&\qquad\cdot \vc\circ\big[2\vg\circ(\bm{1}-\vg^2)\circ\vw + x(\bm{1}-3\vg^2)\circ(\bm{1}-\vg^2)\circ\vw^2\big] \diff{x}\Bigg\rvert\\
&\le \frac{2}{\pi} \Bigg(\int_0^{2\pi}\Big(-2\vc^\top \big[\vg\circ(\bm{1}-\vg^2)\circ\vw^2\big] + a^2\sin(ax)\Big)^2\diff{x} \Bigg)^{\frac{1}{2}}\\
&\qquad\cdot \Bigg(\int_0^{2\pi}\vc^2\circ\big[2\vg\circ(\bm{1}-\vg^2)\circ\vw \\
&\qquad\qquad+ x(\bm{1}-3\vg^2)\circ(\bm{1}-\vg^2)\circ\vw^2\big]^2\diff{x} \Bigg)^{\frac{1}{2}}\\
&\le \frac{2}{\pi} \Bigg(\int_0^{2\pi}\big(2|\vc|^\top\vw^2 + a^2\big)^2\diff{x} \Bigg)^{\frac{1}{2}}\cdot \Bigg(\int_0^{2\pi}\vc^2\circ(\vw+2\pi\vw^2)^2\diff{x} \Bigg)^{\frac{1}{2}}\\
&= 4 \big(2|\vc|^\top\vw^2 + a^2\big) \cdot |\vc| \circ (|\vw|+2\pi\vw^2)\\
&\le 16\pi \big(|\vc|^\top\vw^2 + a^2\big) \cdot |\vc| \circ (|\vw|+\vw^2).
\end{aligned}
\end{equation}
Thus, the bound for $(\partial \mathcal{L}_{\mathcal{F}} / \partial \vw)^\top$ is given by 
\begin{equation} \label{eq:bound_w}
    \bigg\lvert \Big(\frac{\partial \mathcal{L}_{\mathcal{F}}}{\partial \vw}\Big)^\top \bigg\rvert =  \mathcal{O}\big(|\vc|^\top\vw^2 + a^2\big)\cdot\big(|\vc| \circ |\vw|\circ( |\vw| + \bm{1} )\big).
\end{equation}

And for $(\partial \mathcal{L}_{\mathcal{F}} / \partial \vb)^\top$, we have
\begin{equation}
\begin{aligned}
\bigg\lvert \Big(\frac{\partial \mathcal{L}_{\mathcal{F}}}{\partial \vb}\Big)^\top \bigg\rvert &= \frac{1}{2\pi}\Bigg\lvert \int_0^{2\pi}2\Big(\frac{\Diff2{\hat{u}}}{\diff{x}^2} + a^2\sin(ax)\Big) \cdot \bigg(\partial \Big( \frac{\Diff2{\hat{u}}}{\diff{x}^2} \Big) \Big/ \partial \vb \bigg) \diff{x}\Bigg\rvert\\
&= \frac{2}{\pi}\Bigg\lvert\int_0^{2\pi}\Big(-2\vc^\top \big[\vg\circ(\bm{1}-\vg^2)\circ\vw^2\big] + a^2\sin(ax)\Big)\\ 
&\qquad\cdot \big[\vc\circ (\bm{1}-3\vg^2)\circ(\bm{1}-\vg^2)\circ\vw^2\big] \diff{x}\Bigg\rvert\\
&\le \frac{2}{\pi} \Bigg(\int_0^{2\pi}\Big(-2\vc^\top \big[\vg\circ(\bm{1}-\vg^2)\circ\vw^2\big] + a^2\sin(ax)\Big)^2\diff{x} \Bigg)^{\frac{1}{2}}\\
&\qquad\cdot \Bigg(\int_0^{2\pi}\big[\vc\circ (\bm{1}-3\vg^2)\circ(\bm{1}-\vg^2)\circ\vw^2\big]^2\diff{x} \Bigg)^{\frac{1}{2}}\\
&\le \frac{2}{\pi} \Bigg(\int_0^{2\pi}\big(2|\vc|^\top\vw^2 + a^2\big)^2\diff{x} \Bigg)^{\frac{1}{2}}\cdot \Bigg(\int_0^{2\pi}\vc^2\circ\vw^4\diff{x} \Bigg)^{\frac{1}{2}}\\
&= 4 \big(2|\vc|^\top\vw^2 + a^2\big) \cdot |\vc| \circ \vw^2\\
&\le 8 \big(|\vc|^\top\vw^2 + a^2\big) \cdot |\vc| \circ \vw^2.
\end{aligned}
\end{equation}
Thus, the bound for $(\partial \mathcal{L}_{\mathcal{F}} / \partial \vb)^\top$ is given by 
\begin{equation} \label{eq:bound_b}
    \bigg\lvert \Big(\frac{\partial \mathcal{L}_{\mathcal{F}}}{\partial \vb}\Big)^\top \bigg\rvert =  \mathcal{O}\big(|\vc|^\top\vw^2 + a^2\big)\cdot\big(|\vc| \circ \vw^2\big).
\end{equation}

Recalling that $\bm{\theta} = (\vc, \vw, \vb)$, from Eq.~\eqref{eq:bound_c}, Eq.~\eqref{eq:bound_w}, and Eq.~\eqref{eq:bound_b}, we have
\begin{equation} \label{eq:bound_origin}
    \Big|\big(\nabla_{\bm{\theta}}\mathcal{L}_{\mathcal{F}}\big)^\top\Big|=\bigg\lvert \Big(\frac{\partial \mathcal{L}_{\mathcal{F}}}{\partial \bm{\theta}}\Big)^\top \bigg\rvert =  \mathcal{O}\big(|\vc|^\top\vw^2 + a^2\big)\cdot\big(\vw^2, |\vc| \circ |\vw|\circ( |\vw| + \bm{1} ), |\vc| \circ \vw^2\big).
\end{equation}

In contrast, for the transformed PDE,  we first derive the derivatives of the ansatz (i.e., $\hat{u}=\vc^\top \sigma(\vw x+\vb), \hat{p}=\vc_p^\top \sigma(\vw x+\vb)$)
\begin{subequations}
\begin{align}
    \frac{\diff{\hat{u}}}{\diff{x}} &= \vc^\top \big[(\bm{1}-\sigma^2(\vw x+\vb))\circ\vw\big],\\
    \frac{\diff{\hat{p}}}{\diff{x}} &= \vc_p^\top \big[(\bm{1}-\sigma^2(\vw x+\vb))\circ\vw\big].
\end{align}
\end{subequations}
We again abbreviate $\sigma(\vw x+\vb)$ as $\vg$ to obtain
\begin{subequations}
\begin{align}
    \frac{\diff{\hat{u}}}{\diff{x}} &= \vc^\top \big[(\bm{1}-\vg^2)\circ\vw\big],\\
    \frac{\diff{\hat{p}}}{\diff{x}} &= \vc_p^\top \big[(\bm{1}-\vg^2)\circ\vw\big].
\end{align}
\end{subequations}

We now compute a bound for $(\partial \tilde{\mathcal{L}}_{\mathcal{F}} / \partial \vc)^\top$
\begin{equation}
\begin{aligned}
\bigg\lvert \Big(\frac{\partial \tilde{\mathcal{L}}_{\mathcal{F}}}{\partial \vc}\Big)^\top \bigg\rvert &= \frac{1}{2\pi}\Bigg\lvert \int_0^{2\pi}2\Big(\hat{p} - \frac{\diff{\hat{u}}}{\diff{x}}\Big) \cdot \bigg(\partial \Big( \frac{\diff{\hat{u}}}{\diff{x}} \Big) \Big/ \partial \vc \bigg) \diff{x}\Bigg\rvert\\
&= \frac{1}{\pi}\Bigg\lvert\int_0^{2\pi}\Big(\vc_p^\top \vg - \vc^\top \big[(\bm{1}-\vg^2)\circ\vw\big] \Big) \cdot \big[(\bm{1}-\vg^2)\circ\vw\big] \diff{x}\Bigg\rvert\\
&\le \frac{1}{\pi} \Bigg(\int_0^{2\pi}\Big(\vc_p^\top \vg - \vc^\top \big[(\bm{1}-\vg^2)\circ\vw\big]\Big)^2\diff{x} \Bigg)^{\frac{1}{2}}\cdot \Bigg(\int_0^{2\pi}\big[(\bm{1}-\vg^2)\circ\vw\big]^2\diff{x} \Bigg)^{\frac{1}{2}}\\
&\le \frac{1}{\pi} \Bigg(\int_0^{2\pi}\big(\lVert \vc_p \rVert_1 + |\vc|^\top |\vw|\big)^2\diff{x} \Bigg)^{\frac{1}{2}}\cdot \Bigg(\int_0^{2\pi}\vw^2\diff{x} \Bigg)^{\frac{1}{2}}\\
&= 2 \big(\lVert \vc_p \rVert_1 + |\vc|^\top |\vw|\big) |\vw|,
\end{aligned}
\end{equation}
Thus, the bound for $(\partial \tilde{\mathcal{L}}_{\mathcal{F}} / \partial \vc)^\top$ is given by 
\begin{equation} \label{eq:bound_c_transformed}
    \bigg\lvert \Big(\frac{\partial \tilde{\mathcal{L}}_{\mathcal{F}}}{\partial \vc}\Big)^\top \bigg\rvert =  \mathcal{O}\big(\lVert \vc_p \rVert_1 + |\vc|^\top |\vw|\big)\cdot|\vw|.
\end{equation}

As for $(\partial \tilde{\mathcal{L}}_{\mathcal{F}} / \partial \vw)^\top$, we have
\begin{equation}
\begin{aligned}
&\bigg\lvert \Big(\frac{\partial \tilde{\mathcal{L}}_{\mathcal{F}}}{\partial \vw}\Big)^\top \bigg\rvert\\
=& \frac{1}{2\pi}\Bigg\lvert \int_0^{2\pi}2\Big(\frac{\diff{\hat{p}}}{\diff{x}} + a^2\sin(ax)\Big) \cdot \bigg(\partial \Big( \frac{\diff{\hat{p}}}{\diff{x}} \Big) \Big/ \partial \vw \bigg) \diff{x} + \int_0^{2\pi}2\Big(\hat{p} - \frac{\diff{\hat{u}}}{\diff{x}}\Big) \cdot \bigg(\frac{\partial \hat{p}}{\partial \vw} - \partial\Big( \frac{\diff{\hat{u}}}{\diff{x}} \Big) \Big/ \partial \vw \bigg) \diff{x}\Bigg\rvert\\
=& \frac{1}{\pi}\Bigg\lvert \int_0^{2\pi}\Big(\vc_p^\top \big[(\bm{1}-\vg^2)\circ\vw\big] + a^2\sin(ax)\Big)\cdot \big[ \vc_p\circ(\bm{1}-\vg^2)\circ(\bm{1}-2x\vg\circ\vw) \big]\diff{x}\\
& + \int_0^{2\pi}\Big(\vc_p^\top \vg - \vc^\top \big[(\bm{1}-\vg^2)\circ\vw\big]\Big) \cdot \big[x\vc_p\circ(\bm{1}-\vg^2) - \vc\circ(\bm{1}-\vg^2)\circ(\bm{1}-2x\vg\circ\vw) \big] \diff{x}\Bigg\rvert\\
\le& \frac{1}{\pi} \Bigg(\bigg(\int_0^{2\pi}\Big(\vc_p^\top \big[(\bm{1}-\vg^2)\circ\vw\big] + a^2\sin(ax)\Big)^2\diff{x} \bigg)^{\frac{1}{2}}\\
&\qquad\qquad\cdot \bigg(\int_0^{2\pi}\big[ \vc_p\circ(\bm{1}-\vg^2)\circ(\bm{1}-2x\vg\circ\vw)\big]^2\diff{x} \bigg)^{\frac{1}{2}}\\
&\qquad+\bigg(\int_0^{2\pi}\Big(\vc_p^\top \vg - \vc^\top \big[(\bm{1}-\vg^2)\circ\vw\big]\Big)^2\diff{x} \bigg)^{\frac{1}{2}}\\
&\qquad\qquad\cdot \bigg(\int_0^{2\pi}\big[x\vc_p\circ(\bm{1}-\vg^2) - \vc\circ(\bm{1}-\vg^2)\circ(\bm{1}-2x\vg\circ\vw)\big]^2\diff{x} \bigg)^{\frac{1}{2}}\Bigg)\\
&\le 4 \Bigg( \bigg(\int_0^{2\pi}\big(|\vc_p|^\top|\vw| + a^2\big)^2\diff{x} \bigg)^{\frac{1}{2}}\cdot \bigg(\int_0^{2\pi}\vc_p^2\circ(\bm{1}+|\vw|)^2\diff{x} \bigg)^{\frac{1}{2}}\\
&\qquad+\bigg(\int_0^{2\pi}\big(\lVert\vc_p\rVert_1+|\vc|^\top|\vw|)^2\diff{x} \bigg)^{\frac{1}{2}}\cdot \bigg(\int_0^{2\pi}\big[ |\vc_p| + |\vc|\circ(\bm{1}+|\vw|) \big]^2\diff{x} \bigg)^{\frac{1}{2}} \Bigg)\\
&= 8\pi \Big(\big(|\vc_p|^\top|\vw| + a^2\big) \cdot \big[|\vc_p| \circ (\bm{1}+|\vw|)\big] + \big(\lVert\vc_p\rVert_1+|\vc|^\top|\vw|\big) \cdot  \big[|\vc_p| + |\vc|\circ(\bm{1}+|\vw|)\big]\Big)\\
&\le 40\pi \big(\lVert\vc_p\rVert_1 + \max(|\vc|,|\vc_p|)^\top|\vw| + a^2\big) \cdot \big[ \max(|\vc|,|\vc_p|)\circ\max(|\vw|,\bm{1}) \big].
\end{aligned}
\end{equation}
Thus, we can bound $(\partial \tilde{\mathcal{L}}_{\mathcal{F}} / \partial \vw)^\top$ by 
\begin{equation}
\begin{split}
\label{eq:bound_w_transformed}
    \bigg\lvert \Big(\frac{\partial \tilde{\mathcal{L}}_{\mathcal{F}}}{\partial \vw}\Big)^\top \bigg\rvert 
    &\le 40\pi \big(\lVert\vc_p\rVert_1 + \max(|\vc|,|\vc_p|)^\top|\vw| + a^2\big) \cdot \big[ \max(|\vc|,|\vc_p|)\circ\max(|\vw|,\bm{1}) \big]\\
    &=  \mathcal{O}\big(\lVert\vc_p\rVert_1 + \max(|\vc|,|\vc_p|)^\top|\vw| + a^2\big)\cdot\big[ \max(|\vc|,|\vc_p|)\circ\max(|\vw|,\bm{1}) \big].
\end{split}
\end{equation}

And for $(\partial \tilde{\mathcal{L}}_{\mathcal{F}} / \partial \vb)^\top$, we have
\begin{equation}
\begin{aligned}
& \bigg\lvert \Big(\frac{\partial \tilde{\mathcal{L}}_{\mathcal{F}}}{\partial \vb}\Big)^\top \bigg\rvert\\
=& \frac{1}{2\pi}\Bigg\lvert \int_0^{2\pi}2\Big(\frac{\diff{\hat{p}}}{\diff{x}} + a^2\sin(ax)\Big) \cdot \bigg(\partial \Big( \frac{\diff{\hat{p}}}{\diff{x}} \Big) \Big/ \partial \vb \bigg) \diff{x}\\
&\qquad + \int_0^{2\pi}2\Big(\hat{p} - \frac{\diff{\hat{u}}}{\diff{x}}\Big) \cdot \bigg(\frac{\partial \hat{p}}{\partial \vb} - \partial\Big( \frac{\diff{\hat{u}}}{\diff{x}} \Big) \Big/ \partial \vb \bigg) \diff{x}\Bigg\rvert\\
=& \frac{1}{\pi}\Bigg\lvert \int_0^{2\pi}\Big(\vc_p^\top \big[(\bm{1}-\vg^2)\circ\vw\big] + a^2\sin(ax)\Big)\cdot \big[ -2\vc_p\circ(\bm{1}-\vg^2)\circ\vg\circ\vw \big]\diff{x}\\
&+ \int_0^{2\pi}\Big(\vc_p^\top \vg - \vc^\top \big[(\bm{1}-\vg^2)\circ\vw\big]\Big)\cdot \big[\vc_p\circ(\bm{1}-\vg^2) + 2\vc\circ(\bm{1}-\vg^2)\circ\vg\circ\vw \big] \diff{x}\Bigg\rvert\\
\le& \frac{1}{\pi} \Bigg(\bigg(\int_0^{2\pi}\Big(\vc_p^\top \big[(\bm{1}-\vg^2)\circ\vw\big] + a^2\sin(ax)\Big)^2\diff{x} \bigg)^{\frac{1}{2}}\\
&\qquad\qquad\cdot \bigg(\int_0^{2\pi} \big[ -2\vc_p\circ(\bm{1}-\vg^2)\circ\vg\circ\vw \big]^2\diff{x} \bigg)^{\frac{1}{2}}\\
&+\bigg(\int_0^{2\pi}\Big(\vc_p^\top \vg - \vc^\top \big[(\bm{1}-\vg^2)\circ\vw\big]\Big)^2\diff{x} \bigg)^{\frac{1}{2}}\\
&\qquad\qquad\cdot \bigg(\int_0^{2\pi}\big[\vc_p\circ(\bm{1}-\vg^2) + 2\vc\circ(\bm{1}-\vg^2)\circ\vg\circ\vw]^2\diff{x} \bigg)^{\frac{1}{2}}\Bigg)\\
\le& \frac{2}{\pi} \Bigg( \bigg(\int_0^{2\pi}\big(|\vc_p|^\top|\vw| + a^2\big)^2\diff{x} \bigg)^{\frac{1}{2}}\cdot \bigg(\int_0^{2\pi}\vc_p^2\circ\vw^2\diff{x} \bigg)^{\frac{1}{2}}\\
&+\bigg(\int_0^{2\pi}\big(\lVert\vc_p\rVert_1+|\vc|^\top|\vw|)^2\diff{x} \bigg)^{\frac{1}{2}}\cdot \bigg(\int_0^{2\pi}\big[ |\vc_p| + |\vc|\circ|\vw| \big]^2\diff{x} \bigg)^{\frac{1}{2}} \Bigg)\\
=& 4 \Big(\big(|\vc_p|^\top|\vw| + a^2\big) \cdot \big[|\vc_p| \circ |\vw|\big]+ \big(\lVert\vc_p\rVert_1+|\vc|^\top|\vw|\big) \cdot  \big[|\vc_p| + |\vc|\circ|\vw|\big]\Big)\\
\le& 12 \big(\lVert\vc_p\rVert_1 + \max(|\vc|,|\vc_p|)^\top|\vw| + a^2\big) \cdot \big[ \max(|\vc|,|\vc_p|)\circ\max(|\vw|,\bm{1}) \big].
\end{aligned}
\end{equation}
Thus, we can bound $(\partial \tilde{\mathcal{L}}_{\mathcal{F}} / \partial \vb)^\top$ by 
\begin{equation}
\begin{split}
\label{eq:bound_b_transformed}
    \bigg\lvert \Big(\frac{\partial \tilde{\mathcal{L}}_{\mathcal{F}}}{\partial \vb}\Big)^\top \bigg\rvert 
    \le& 12 \big(\lVert\vc_p\rVert_1 + \max(|\vc|,|\vc_p|)^\top|\vw| + a^2\big) \cdot \big[ \max(|\vc|,|\vc_p|)\circ\max(|\vw|,\bm{1}) \big]\\
    = & \mathcal{O}\big(\lVert\vc_p\rVert_1 + \max(|\vc|,|\vc_p|)^\top|\vw| + a^2\big) \cdot \big[ \max(|\vc|,|\vc_p|)\circ\max(|\vw|,\bm{1}) \big].
\end{split}
\end{equation}

For $(\partial \tilde{\mathcal{L}}_{\mathcal{F}} / \partial \vc_p)^\top$, we have
\begin{equation}
\begin{aligned}
\bigg\lvert &\Big(\frac{\partial \tilde{\mathcal{L}}_{\mathcal{F}}}{\partial \vc_p}\Big)^\top \bigg\rvert \\ 
=& \frac{1}{2\pi}\Bigg\lvert \int_0^{2\pi}2\Big(\frac{\diff{\hat{p}}}{\diff{x}} + a^2\sin(ax)\Big) \cdot \bigg(\partial \Big( \frac{\diff{\hat{p}}}{\diff{x}} \Big) \Big/ \partial \vc_p \bigg) \diff{x}
+ \int_0^{2\pi}2\Big(\hat{p} - \frac{\diff{\hat{u}}}{\diff{x}}\Big) \cdot \bigg(\frac{\partial \hat{p}}{\partial \vc_p} \bigg) \diff{x}\Bigg\rvert\\
=& \frac{1}{\pi}\Bigg\lvert \int_0^{2\pi}\Big(\vc_p^\top \big[(\bm{1}-\vg^2)\circ\vw\big] + a^2\sin(ax)\Big)\cdot \big[ (\bm{1}-\vg^2)\circ\vw \big]\diff{x}\\
&\qquad + \int_0^{2\pi}\Big(\vc_p^\top \vg - \vc^\top \big[(\bm{1}-\vg^2)\circ\vw\big]\Big)\cdot\vg \diff{x}\Bigg\rvert\\
\le& \frac{1}{\pi} \Bigg(\bigg(\int_0^{2\pi}\Big(\vc_p^\top \big[(\bm{1}-\vg^2)\circ\vw\big] + a^2\sin(ax)\Big)^2\diff{x} \bigg)^{\frac{1}{2}} 
\cdot \bigg(\int_0^{2\pi}\big[(\bm{1}-\vg^2)\circ\vw)\big]^2\diff{x} \bigg)^{\frac{1}{2}}\\
&\qquad+\bigg(\int_0^{2\pi}\Big(\vc_p^\top \vg - \vc^\top \big[(\bm{1}-\vg^2)\circ\vw\big]\Big)^2\diff{x} \bigg)^{\frac{1}{2}}\cdot \bigg(\int_0^{2\pi}\vg^2\diff{x} \bigg)^{\frac{1}{2}}\Bigg)\\
\le& \frac{1}{\pi} \Bigg( \bigg(\int_0^{2\pi}\big(|\vc_p|^\top|\vw| + a^2\big)^2\diff{x} \bigg)^{\frac{1}{2}}\cdot \bigg(\int_0^{2\pi}\vw^2\diff{x} \bigg)^{\frac{1}{2}}\\
&\qquad+\bigg(\int_0^{2\pi}\big(\lVert\vc_p\rVert_1+|\vc|^\top|\vw|)^2\diff{x} \bigg)^{\frac{1}{2}}\cdot \bigg(\int_0^{2\pi}\bm{1}\diff{x} \bigg)^{\frac{1}{2}} \Bigg)\\
=& 2 \Big(\big(|\vc_p|^\top|\vw| + a^2\big) \cdot |\vw| + \big(\lVert\vc_p\rVert_1+|\vc|^\top|\vw|\big) \cdot  \bm{1}\Big)\\
\le& 4 \big(\lVert\vc_p\rVert_1 + \max(|\vc|,|\vc_p|)^\top|\vw| + a^2\big) \cdot \max(|\vw|,\bm{1}).
\end{aligned}
\end{equation}
Thus, we can bound $(\partial \tilde{\mathcal{L}}_{\mathcal{F}} / \partial \vc_p)^\top$ by 
\begin{equation}
\begin{split}
\label{eq:bound_c_p_transformed}
    \left\lvert \Big(\frac{\partial \tilde{\mathcal{L}}_{\mathcal{F}}}{\partial \vc_p}\Big)^\top \right\rvert
    \le& 4 \left(\lVert\vc_p\rVert_1 + \max(|\vc|,|\vc_p|)^\top|\vw| + a^2\right) \cdot \max(|\vw|,\bm{1})\\
    = & \mathcal{O}\big(\lVert\vc_p\rVert_1 + \max(|\vc|,|\vc_p|)^\top|\vw| + a^2\big) \cdot \max(|\vw|,\bm{1}).
\end{split}
\end{equation}

Recalling that $\tilde{\bm{\theta}} = (\vc, \vw, \vb, \vc_p)$, from Eq.~\eqref{eq:bound_c_transformed}, Eq.~\eqref{eq:bound_w_transformed}, Eq.~\eqref{eq:bound_b_transformed}, and Eq.~\eqref{eq:bound_c_p_transformed}, noting that $\lVert\vc_p\rVert_1 + |\vc|^\top|\vw|\le \lVert\vc_p\rVert_1 + \max(|\vc|,|\vc_p|)^\top|\vw| + a^2$, we have
\begin{equation}\label{eq:bound_transformed}
\begin{aligned}
    \Big|\big({\nabla_{\tilde{\bm{\theta}}}\tilde{\mathcal{L}}_{\mathcal{F}}}\big)^\top\Big|&=\bigg\lvert \Big(\frac{\partial \tilde{\mathcal{L}}_{\mathcal{F}}}{\partial \tilde{\bm{\theta}}}\Big)^\top \bigg\rvert =  \mathcal{O}\big(\lVert\vc_p\rVert_1 + \max(|\vc|,|\vc_p|)^\top|\vw| + a^2\big)\cdot \big(
    |\vw|, \\
    &\max(|\vc|,|\vc_p|)\circ\max(|\vw|, \bm{1}),\max(|\vc|,|\vc_p|)\circ\max(|\vw|,\bm{1}), \max(|\vw|,\bm{1})
    \big).
\end{aligned}
\end{equation}

\subsubsection{Analysis via the Condition Number}\label{sec_a7_3}
In addition to providing bounds for the gradients of $\mathcal{L}_{\mathcal{F}}$ and $\tilde{\mathcal{L}}_{\mathcal{F}}$, we can also use the condition number to characterize the sensitivity of their gradients with respect to the parameters of the neural network. Let $\bm{\theta}^{(t)}=(\vc^{(t)},\vw^{(t)},\vb^{(t)})$ and $\tilde{\bm{\theta}}^{(t)}=(\vc^{(t)},\vw^{(t)},\vb^{(t)},\vc_p^{(t)})$ are the parameters of the neural network in the $t$th step (before and after the reformulation). For simplicity, we introduce the following notations
\begin{subequations}
\begin{align}
    \Delta \bm{\theta} &= \bm{\theta}^{(t+1)} - \bm{\theta}^{(t)},\\
    \Delta \tilde{\bm{\theta}} &= \tilde{\bm{\theta}}^{(t+1)} - \tilde{\bm{\theta}}^{(t)},\\
    \Delta \mathcal{L}_{\mathcal{F}} &= \mathcal{L}_{\mathcal{F}}(\bm{\theta}^{(t+1)}) - \mathcal{L}_{\mathcal{F}}(\bm{\theta}^{(t)}),\\
    \Delta \tilde{\mathcal{L}}_{\mathcal{F}} &= \tilde{\mathcal{L}}_{\mathcal{F}}(\tilde{\bm{\theta}}^{(t+1)}) - \tilde{\mathcal{L}}_{\mathcal{F}}(\tilde{\bm{\theta}}^{(t)}).
\end{align}
\end{subequations}
The condition numbers of $\mathcal{L}_{\mathcal{F}}$ and $\tilde{\mathcal{L}}_{\mathcal{F}}$ are defined as
\begin{equation}
    \mathrm{cond} = \frac{|\Delta \mathcal{L}_{\mathcal{F}}|}{ \lVert\Delta \bm{\theta}\rVert_2},\quad 
    \tilde{\mathrm{cond}} = \frac{|\Delta \tilde{\mathcal{L}}_{\mathcal{F}}|}{ \lVert\Delta \tilde{\bm{\theta}}\rVert_2}.
\end{equation}
Next we derive the bounds for $\mathrm{cond}$ and $\tilde{\mathrm{cond}}$, respectively. We first consider $\mathrm{cond}$
\begin{equation}
\begin{aligned}
     \mathrm{cond} &= \frac{|\Delta \mathcal{L}_{\mathcal{F}}|}{ \lVert\Delta \bm{\theta}\rVert_2}
     \approx\bigg\lvert\Big(\frac{\partial \mathcal{L}_{\mathcal{F}}}{\partial \bm{\theta}}\Big)^\top \cdot \Delta\bm{\theta} \bigg\rvert \bigg / \lVert\Delta \bm{\theta}\rVert_2
     \le \bigg\lVert\Big(\frac{\partial \mathcal{L}_{\mathcal{F}}}{\partial \bm{\theta}}\Big)^\top \bigg\rVert_2\\
     &= \mathcal{O}\Big(\big(|\vc|^\top\vw^2 + a^2\big)\cdot\big\lVert\vw^2, |\vc| \circ |\vw|\circ( |\vw| + \bm{1} ), |\vc| \circ \vw^2\big\rVert_2\Big)\quad\text{(Eq.~\eqref{eq:bound_origin})}\\
     &=\mathcal{O}\Big(\big(|\vc|^\top\vw^2 + a^2\big)\cdot\big\lVert\vw^2, |\vc| \circ |\vw|\circ( |\vw| + \bm{1} ), |\vc| \circ \vw^2\big\rVert_1\Big)\quad\text{(Equivalence of norms)}\\
     &= \mathcal{O}\Big(\big(|\vc|^\top\vw^2 + a^2\big)\cdot\big( \lVert\vw^2\rVert_1 + \lVert|\vc| \circ |\vw|\circ( |\vw| + \bm{1} )\rVert_1 + \lVert|\vc| \circ \vw^2\rVert_1\big)  \Big)\\
     &= \mathcal{O}\Big(\big(|\vc|^\top\vw^2 + a^2\big)\cdot\big\lVert \max(|\vc|,\bm{1})\circ|\vw|\circ\max(|\vw|,\bm{1}) \big\rVert_1  \Big).
\end{aligned}
\label{eq:cond_before}
\end{equation}
Similarly, for $\tilde{\mathrm{cond}}$, we have
\begin{equation}
\begin{aligned}
     \tilde{\mathrm{cond}} &= \frac{|\Delta \tilde{\mathcal{L}}_{\mathcal{F}}|}{ \lVert\Delta \tilde{\bm{\theta}}\rVert_2} 
     \approx\bigg\lvert\Big(\frac{\partial \tilde{\mathcal{L}}_{\mathcal{F}}}{\partial \tilde{\bm{\theta}}}\Big)^\top \cdot \Delta\tilde{\bm{\theta}} \bigg\rvert \bigg / \lVert\Delta \tilde{\bm{\theta}}\rVert_2
     \le \bigg\lVert\Big(\frac{\partial \tilde{\mathcal{L}}_{\mathcal{F}}}{\partial \tilde{\bm{\theta}}}\Big)^\top \bigg\rVert_2\\
     &= \mathcal{O}\Big(\big(\lVert\vc_p\rVert_1 + \max(|\vc|,|\vc_p|)^\top|\vw| + a^2\big)\cdot \big\lVert
    |\vw|, \max(|\vc|,|\vc_p|)\circ\max(|\vw|, \bm{1}),\\
    &\qquad\max(|\vc|,|\vc_p|)\circ\max(|\vw|,\bm{1}), \max(|\vw|,\bm{1})
    \big\rVert_2\Big)\qquad\text{(Eq.~\eqref{eq:bound_transformed})}\\
     &=\mathcal{O}\Big(\big(\lVert\vc_p\rVert_1 + \max(|\vc|,|\vc_p|)^\top|\vw| + a^2\big)\cdot \big\lVert
    |\vw|, \max(|\vc|,|\vc_p|)\circ\max(|\vw|, \bm{1}),\\
    &\qquad\max(|\vc|,|\vc_p|)\circ\max(|\vw|,\bm{1}), \max(|\vw|,\bm{1})
    \big\rVert_1\Big)\qquad\text{(Equivalence of norms)}\\
     &=\mathcal{O}\Big(\big(\lVert\vc_p\rVert_1 + \max(|\vc|,|\vc_p|)^\top|\vw| + a^2\big)\cdot \big(
    \lVert\vw\rVert_1\\
    &\qquad+ \lVert\max(|\vc|,|\vc_p|)\circ\max(|\vw|, \bm{1})\rVert_1+\lVert\max(|\vc|,|\vc_p|)\circ\max(|\vw|,\bm{1})\rVert_1\\
    &\qquad+\lVert\max(|\vw|,\bm{1})
    \rVert_1\big)\Big)\\
     &= \mathcal{O}\Big(\big(\lVert\vc_p\rVert_1 + \max(|\vc|,|\vc_p|)^\top|\vw| + a^2\big)\cdot\big\lVert \max(|\vc|,|\vc_p|,\bm{1})\circ\max(|\vw|,\bm{1}) \big\rVert_1  \Big).
\end{aligned}
\label{eq:cond_after}
\end{equation}

From Eq.~\eqref{eq:cond_before} and Eq.~\eqref{eq:cond_after}, we find that for the original PDEs, the condition number has a higher order relationship with respect to $\bm{\theta}$. If $\bm{\theta}$ is large, the condition number can be very large, leading to oscillations in training. However, if $\bm{\theta}$ is small, the condition number will also be very small, resulting in smaller changes of $\bm{\theta}$ between adjacent iterations and therefore slower convergence. In contrast, after the reformulation, the condition number has a lower order relationship with respect to $\tilde{\bm{\theta}}$, which keeps the condition number more stable during training and alleviates this problem. 

\subsection{Experimental details}\label{sec_a8}
In the following, we will briefly introduce some essential details of our experiments, including the experimental environment, hyper-parameters, construction of the ansatz, and details of the governing PDEs in each experiment. We first introduce the experimental environment, while other details are put into the subsections corresponding to the experiments.

\paragraph{Experimental environment}
We use PyTorch \citep{pytorch2019} as our deep learning library. And our codes for the physics-informed learning are based on DeepXDE \citep{lu2021deepxde}. We train all the models except domain decomposition based baselines (i.e., xPINN, FBPINN, and PFNN-2) on one NVIDIA TITAN Xp 12GB GPU, while the other three are trained on eight NVIDIA GeForce RTX 3090 24GB GPUs (since domain decomposition based models consist of several sub-networks and require more memory to be stored). The operating system is Ubuntu 18.04.5 LTS. If the analytical solution is unavailable, the ground truth solutions to the PDEs (i.e., the testing data) will be generated by COMSOL Multiphysics, a FEM commercial software. And we have put the generated testing data into the zip file.

\subsubsection{Simulation of a 2D battery pack (Heat Equation)}\label{sec_a8_1}
\paragraph{Governing PDEs}
The governing PDEs (along with boundary/initial conditions) are given by
\begin{subequations}
\label{eq:pde_battery}
\begin{align}
    \frac{\partial T}{\partial t}&=k\Delta T(\vx,t),&&\boldsymbol x\in \Omega,t\in(0,1], \\
    k\big(\vn(\vx) \cdot \nabla T(\vx,t)\big) &=h\big(T_a - T(\vx,t)\big),&&\boldsymbol x\in \gamma_{\mathrm{outer}},t\in(0,1], \\
    k\big(\vn(\vx) \cdot \nabla T(\vx,t)\big) &=h\big(T_c - T(\vx,t)\big),&&\boldsymbol x\in \gamma_{\mathrm{cell},i},t\in(0,1],&i=1,\dots,n_c, \\
    k\big(\vn(\vx) \cdot \nabla T(\vx,t)\big) &=h\big(T_w - T(\vx,t)\big),&&\boldsymbol x\in \gamma_{\mathrm{pipe},i},t\in(0,1],&i=1,\dots,n_w, \\
    T(\boldsymbol x,0)&=T_0,&&\boldsymbol x\in \Omega, 
\end{align}
\end{subequations}
where $\vx=(x_1,x_2)$, $t$ are the spatial and temporal coordinates, respectively, $T(\vx,t)$ is the temperature over time, $k=1$ is the thermal conductivity, $\Delta T=\partial^2 T / \partial x_1^2 + \partial^2 T / \partial x_2^2$, $h=1$ is the heat transfer coefficient, $\nabla T=(\partial T / \partial x_1, \partial T / \partial x_2)$, $T_a=0.1$, $T_c=5$, $T_w=1$ are, respectively, the temperature of the air, the cells ($n_c=11$ cells of radius $r_c=1$), the cooling pipes ($n_w=6$ pipes of radius $r_w=0.4$), $T_0=0.1$ is the initial temperature, and the geometry (i.e., $\Omega$, $\gamma_{\mathrm{outer}}$, etc) is shown in Figure~3(a).

And the reformulated PDEs are (which is used by the proposed model, HC)
\begin{subequations}
\begin{align}
    \frac{\partial T}{\partial t}&=k\big(\nabla \cdot \vp (\vx, t)\big),&&\boldsymbol x\in \Omega,t\in(0,1], \\
    \vp(\vx, t) &= \nabla T,&&\boldsymbol x\in \Omega\cup\partial\Omega,t\in(0,1],\\
    k\big(\vn(\vx)\cdot\vp(\vx, t)\big) &=h\big(T_a - T(\vx,t)\big),&&\boldsymbol x\in \gamma_{\mathrm{outer}},t\in(0,1],\label{eq:robin_1}\\
    k\big(\vn(\vx)\cdot\vp(\vx,t)\big) &=h\big(T_c - T(\vx,t)\big),&&\boldsymbol x\in \gamma_{\mathrm{cell},i},t\in(0,1],&i=1,\dots,n_c, \label{eq:robin_2}\\
    k\big(\vn(\vx)\cdot\vp(\vx,t)\big) &=h\big(T_w - T(\vx,t)\big),&&\boldsymbol x\in \gamma_{\mathrm{pipe},i},t\in(0,1],&i=1,\dots,n_w, \label{eq:robin_3}\\
    T(\boldsymbol x,0)&=T_0,&&\boldsymbol x\in\Omega, 
\end{align}
\end{subequations}
where $\vp(\vx,t)$ is the introduced extra field.

\paragraph{Construction of the Ansatz}
Since the solution is a scalar function (i.e., $T(\vx,t)$), we directly denote the solution by $u(\vx, t)=T(\vx,t)$. Let $\tilde{\vp}$ denote $(u, \vp)$. We first derive the general solutions at $\gamma_{\mathrm{outer}}, \gamma_{\mathrm{cell},i}, \gamma_{\mathrm{pipe},i}$, respectively.

For $\vx\in \gamma_{\mathrm{outer}}$, we have $a(\vx) = h$, $b(\vx)=k$, $g(\vx)=hT_a$. According to Eq.~(10), the general solution $\tilde{\vp}^{\gamma_{\mathrm{outer}}}$ is given by
\begin{equation}
    \tilde{\vp}^{\gamma_{\mathrm{outer}}}=\bm{B}(\vx)\mathrm{NN}^{\gamma_{\mathrm{outer}}}(\vx, t) + \frac{(h,k\bm{n})}{\sqrt{h^2 + k^2}}\frac{hT_a}{\sqrt{h^2 + k^2}},
\end{equation}
where $\bm{B}(\vx)$ is computed in Eq.~\eqref{eq:accept_example} (with $\tilde{\vn} = (\tilde{n}_1,\tilde{n}_2,\tilde{n}_3)=(h,k\bm{n})/{\sqrt{h^2 + k^2}}$). And for $\vx \in \gamma_{\mathrm{cell},i}$ and $\gamma_{\mathrm{pipe},i}$, the derivation is similar, where we only need to change $T_a$ to $T_c$ and $T_w$, respectively.

Then, we gather all the general solutions computed to form our ansatz $(\hat{u}, \hat{\vp})$ according to Eq.~\eqref{eq_time_ansatz}, where $\{ \gamma_{\mathrm{outer}}, \gamma_{\mathrm{cell},1},\dots,\gamma_{\mathrm{cell},n_c},\gamma_{\mathrm{pipe},1},\dots,\gamma_{\mathrm{pipe},n_w} \}$ are reordered as $\{ \gamma_{i} \}_{i=1}^{1+n_c+n_w}$ and $f(\vx) = T_0$.

\paragraph{Choices of Extended Distance Functions}
For $\gamma_{\mathrm{cell},i}$ and $\gamma_{\mathrm{pipe},i}$, since they are 2D circles, we can directly choose the extended distance functions $l^{\gamma_{\mathrm{cell},i}}(\vx)$ and $l^{\gamma_{\mathrm{pipe},i}}(\vx)$ as the distance between $\vx$ and the center minus the radius. For the rectangular $\gamma_{\mathrm{outer}}$, supposing that it is given by $[a_1, a_2]\times[b_1,b_2]$, we construct the extended distance function $l^{\gamma_{\mathrm{outer}}}$ as follows
\begin{equation}\label{eq_ext_dist_rec}
    l^{\gamma_{\mathrm{outer}}}(\vx) = \mathrm{SoftMin} (x_1 - a_1, a_2 - x_1, x_2 - b_1, b_2 - x_2),
\end{equation}
where $\vx=(x_1,x_2)$, $\mathrm{SoftMin}$ is a differentiable version of min function which is implemented by \texttt{LogSumExp} in PyTorch, i.e., $\mathrm{SoftMin}(\bm{y}) = \mathrm{LogSumExp}(-\beta \bm{y}) / (-\beta)$, $\beta=4$. And the extended function $l^{\partial \Omega}(\vx)$ is computed by taking the $\mathrm{SoftMin}$ of the distances to all the boundaries
\begin{equation}
    l^{\partial \Omega}(\vx) = \mathrm{SoftMin}(l^{\gamma_{\mathrm{outer}}}(\vx), l^{\gamma_{\mathrm{cell},1}}(\vx),\dots,l^{\gamma_{\mathrm{cell},n_c}}(\vx),l^{\gamma_{\mathrm{pipe},1}}(\vx),\dots,l^{\gamma_{\mathrm{pipe},n_w}}(\vx)).
\end{equation}

\paragraph{Implementation} 
All the models are trained for 5000 Adam iterations (with a learning rate scheduler of \texttt{ReduceLROnPlateau} from PyTorch and an initial learning rate of 0.01), followed by a L-BFGS optimization until convergence. Unless otherwise specified, the mean squared error (MSE) is used for the loss function and $\tanh$ is used for the activation function. And the hyper-parameters of each model are listed as follow
\begin{itemize}
    \item \textbf{HC}: The main neural network is a multilayer perceptron (MLP) of size $[3] + 4\times[50] + [3]$ (which means 3 inputs, 4 hidden layers of width 50, and 3 outputs). The sub-networks (corresponding to Eq.~\eqref{eq:robin_1}, Eq.~\eqref{eq:robin_2}, and Eq.~\eqref{eq:robin_3}) are all MLPs of size $[3] + 3\times[20] + [3]$. And the hyper-parameters of ``hardness'' are $\beta_s=5$ and $\beta_t=10$.
    \item \textbf{PINN}: The ansatz is an MLP of size $[3] + 4\times[50] + [1]$.
    \item \textbf{PINN-LA}: The weights of the loss terms corresponding to the BCs are approximated by $\hat{\lambda}_i = \max_{\theta_n}\{ |\nabla_\theta \mathcal{L}_r(\theta_n)| \} \Big/ \overline{|\nabla_\theta \lambda_i \mathcal{L}_i(\theta_n)|}$. And the parameter of the moving average is $\alpha=0.1$, which is recommended by the paper \citep{wang2021understanding}. Besides, the parameters of the PINN are the same as above.
    \item \textbf{PINN-LA-2}: In our modified version, we approximate the weights of the loss terms as $\hat{\lambda}_i = \overline{ |\nabla_\theta \mathcal{L}_r(\theta_n)| } \Big/ \overline{|\nabla_\theta \lambda_i \mathcal{L}_i(\theta_n)|}$. And the parameter of the moving average is also $\alpha=0.1$. Here we replace the maximum with the mean to make the weights of the loss terms more stable during the training process.
    \item \textbf{FBPINN}: The domain of the problem is divided into $4\times 6=24$ subdomains by a regular grid. The size of the sub-network, an MLP, corresponding to each subdomain is $[3] + 3\times[30] + [1]$. And the scale factor is $\sigma=0.4$, chosen so that the window function is close to zero outside the overlapping region of the subdomains. 
    \item \textbf{xPINN}: The domain of the problem is divided into $4\times 6=24$ subdomains by a regular grid. The size of the sub-network corresponding to each subdomain is $[3] + 3\times[30] + [1]$. And the loss terms of the interface condition include average solution as well as residual continuity conditions.
    \item \textbf{PFNN}: The PFNN considers the variational formulation of the Eq.~\eqref{eq:pde_battery} (i.e., Ritz formulation), and embed the initial condition (IC) into its ansatz (similar to Eq.~(5)). And the size of the neural network (an MLP) is $[3] + 4\times[50] + [1]$.
    \item \textbf{PFNN-2}: The PFNN-2 replaces a single neural network with a domain decomposition based neural network on the basis of PFNN. In the original literature \citep{sheng2022pfnn}, the domain is decomposed in a hard way (like xPINN). However, in our experiments (see Table~1), we find that the performance of hard decomposition is relatively poor, which is because new loss terms are needed to maintain the continuity of the ansatz at the interfaces between the sub-domains, which further aggravates the unbalanced competition. To overcome this, we instead employ a soft domain decomposition, as in FBPINN. See the parts of PFNN and FBPINN for the values of hyper-parameters.
\end{itemize}

\subsubsection{Simulation of an Airfoil (Navier-Stokes Equations)}\label{sec_a8_2}
\paragraph{Governing PDEs}
The governing PDEs are given by
\begin{subequations}
\begin{align}
    \boldsymbol u(\vx) \cdot \nabla \boldsymbol u(\vx) &= -\nabla p(\vx)  + v\nabla^2\boldsymbol u(\vx) ,&&\boldsymbol x\in \Omega,\\
    \nabla \cdot \boldsymbol u(\vx) &= 0,&&\boldsymbol x\in \Omega,\\
    \boldsymbol u(\vx)&=\boldsymbol{u}_0(\vx),&&\boldsymbol x\in \gamma_{\mathrm{inlet}}\cup\gamma_{\mathrm{top}}\cup\gamma_{\mathrm{bottom}},\\
    p(\vx)&=1,&&\boldsymbol x\in \gamma_{\mathrm{outlet}},\\
    \vn(\vx) \cdot \vu(\vx) &=0,&&\boldsymbol x\in\gamma_{\mathrm{airfoil}},
\end{align}
\end{subequations}
where $\vu(\vx) = (u_1(\vx), u_2(\vx))$, $p(\vx)$, $v=1/50$ are the velocity, pressure, and viscosity of the fluid, respectively, $\vu_0(\vx) = (1, 0)$, and the geometry of the problem (i.e., $\Omega$, $\gamma_{\mathrm{inlet}}$, etc) is shown in Figure~3(b).

And the reformulated PDEs are (which is used by the proposed model, HC)
\begin{subequations}
\begin{align}
    \big(\boldsymbol u(\vx) - v\nabla\big) \cdot \big(\vp_1(\vx), \vp_2(\vx)\big) &= -\nabla p(\vx),&&\boldsymbol x\in \Omega,\\
    \nabla \cdot \boldsymbol u(\vx) &= 0,&&\boldsymbol x\in \Omega,\\
    \vp_1(\vx) &= \nabla u_1,&&\boldsymbol x\in \Omega\cup\partial\Omega,\\
    \vp_2(\vx) &= \nabla u_2,&&\boldsymbol x\in \Omega\cup\partial\Omega,\\
    \boldsymbol u(\vx)&=\boldsymbol{u}_0(\vx),&&\boldsymbol x\in \gamma_{\mathrm{inlet}}\cup\gamma_{\mathrm{top}}\cup\gamma_{\mathrm{bottom}},\\
    p(\vx)&=1,&&\boldsymbol x\in \gamma_{\mathrm{outlet}},\\
    \vn(\vx) \cdot \vu(\vx) &=0,&&\boldsymbol x\in\gamma_{\mathrm{airfoil}}, \label{eq:ns_bc}
\end{align}
\end{subequations}
where $\vp_1(\vx)$ and $\vp_2(\vx)$ are the introduced extra fields.

\paragraph{Construction of the Ansatz}
Here, the solution is $(\vu(\vx), p(\vx))$. For $p(\vx)$, the general solution in $\gamma_{\mathrm{outlet}}$ is exactly $p^{\gamma_{\mathrm{outlet}}}(\vx)=1$. And the ansatz for $p$ is given by
\begin{equation}
    \hat{p} = p^{\gamma_{\mathrm{outlet}}}(\vx) + l^{\gamma_{\mathrm{outlet}}}(\vx) \mathrm{NN}^{\mathrm{main}}_j(\vx)[3],
\end{equation}
where $[3]$ means taking the third elements of the output of $\mathrm{NN}^{\mathrm{main}}_j(\vx)$.

For $\vu(\vx)$, the general solution in $\gamma_{\mathrm{inlet}}\cup\gamma_{\mathrm{top}}\cup \gamma_{\mathrm{bottom}}$ is exactly $\vu^{\gamma_{*}}(\vx)=\vu_0(\vx)$, where we define an alias $\gamma_{*}$ for $\gamma_{\mathrm{inlet}}\cup\gamma_{\mathrm{top}}\cup \gamma_{\mathrm{bottom}}$. In $\gamma_{\mathrm{airfoil}}$, the general solution is given by
\begin{equation}
    \vu^{\gamma_{\mathrm{airfoil}}} = \bm{B}(\vx) \mathrm{NN}^{\gamma_{*}}(\vx),
\end{equation}
where $\bm{B}(\vx) = [n_2(\vx), -n_1(\vx)]^\top$ according to Eq.~\eqref{eq:accept_example_2} and the output of $\mathrm{NN}^{\gamma_{*}}(\vx)$ is a scalar. Gathering $\vu^{\gamma_{*}}$ and $\vu^{\gamma_{\mathrm{airfoil}}}$, we then follow Eq.~(11) to obtain the ansatz for $\hat{\vu}$
\begin{equation}
    \hat{\vu} = l^{\partial \Omega}(\vx) \mathrm{NN}^{\mathrm{main}}_j(\vx)[1:2] + \exp\big[-\alpha_{\gamma_{*}} l^{\gamma_{*}}(\vx)\big]\vu^{\gamma_{*}}(\vx) + \exp\big[-\alpha_{\gamma_{\mathrm{airfoil}}} l^{\gamma_{\mathrm{airfoil}}}(\vx)\big]\vu^{\gamma_{\mathrm{airfoil}}}(\vx),
\end{equation}
where $[1:2]$ means taking the first two elements of the output of $\mathrm{NN}^{\mathrm{main}}_j(\vx)$ and $\alpha_{\gamma_{*}}$ as well as $\alpha_{\gamma_{\mathrm{airfoil}}}$ are similarly defined as in Eq.~(12).

\paragraph{Choices of Extended Distance Functions}
For $l^{\gamma_{\mathrm{airfoil}}}(\vx)$, a direct way is to calculate the distance between $\vx$ and the airfoil $\gamma_{\mathrm{airfoil}}$. However, it may be very time-consuming since the $\gamma_{\mathrm{airfoil}}$ is highly complicated. So we prefer to approximate the true distance with an MLP with 3 hidden layers of width 30. We train the neural network before training our main model with $1024 \times 6$ points sampled in $\Omega$ ($5/6$ of them are sampled in the bounding box of the airfoil, and the rest are sampled in $\Omega$) along with their truth distances (which are computed by using the formula of the distance to a polygon) for $10,000$ Adam epochs (with a learning rate scheduler of \texttt{ReduceLROnPlateau} from PyTorch and an initial learning rate of 0.001). The loss function is a $\ell_1$ loss. A polar coordinate transformation trick is utilized as in Appendix~\ref{sec_a5}. 

And for $\gamma_{\mathrm{outlet}}$, since it is a vertical line, for example, $x_1=a$, we can compute the extended distance function as $l^{\gamma_{\mathrm{outlet}}}(\vx) = a - x_1$, where $\vx=(x_1, x_2)$. And $\gamma_*$ is an open rectangle, so we can compute $l^{\gamma_*}(\vx)$ similarly to the case of the rectangle (see Eq.~\eqref{eq_ext_dist_rec}) while ignoring the right side. Besides, $l^{\partial \Omega}(\vx)$ is still computed by taking the $\mathrm{SoftMin}$ of the distances to all the boundaries.

\paragraph{Implementation} 
All the models are trained for 5000 Adam iterations (with a learning rate scheduler of \texttt{ReduceLROnPlateau} from PyTorch and an initial learning rate of 0.001), followed by a L-BFGS optimization until convergence. And the hyper-parameters of each model are listed as follow
\begin{itemize}
    \item \textbf{HC}: The main neural network is an MLP of size $[2] + 6\times[50] + [7]$. The sub-network (corresponding to Eq.~\eqref{eq:ns_bc}) is an MLP of size $[2] + 4\times[40] + [1]$. And the hyper-parameters of ``hardness'' is $\beta_s=5$.
    \item \textbf{PINN}: The ansatz is an MLP of size $[2] + 6\times[50] + [3]$.
    \item \textbf{PINN-LA}: The parameter of the moving average is $\alpha=0.1$.
    \item \textbf{PINN-LA-2}: The parameter of the moving average is $\alpha=0.1$.
    \item \textbf{FBPINN}: The domain of the problem is divided into $3\times 6=18$ subdomains by a regular grid. The size of the sub-network, an MLP, corresponding to each subdomain is $[2] + 4\times[30] + [3]$. And the scale factor is $\sigma=0.2$, chosen so that the window function is close to zero outside the overlapping region of the subdomains. 
    \item \textbf{xPINN}: The domain of the problem is divided into $3\times 6=18$ subdomains by a regular grid. The size of the sub-network corresponding to each subdomain is $[2] + 4\times[30] + [3]$. And the loss terms of the interface condition include average solution as well as residual continuity conditions.
\end{itemize}

\subsubsection{High-dimensional Heat Equation}\label{sec_a8_3}
\paragraph{Governing PDEs}
The governing PDEs are given by
\begin{subequations}
\begin{align}
    \frac{\partial u}{\partial t}&=k\Delta u(\vx,t)+f(\vx, t),&&\boldsymbol x\in \Omega\subset \mathbb{R}^d,t\in(0,1], \\
    \vn(\vx) \cdot \nabla u(\vx,t) &=g(\vx,t),&&\boldsymbol x\in\partial\Omega,t\in(0,1],\\
    u(\vx,0)&=g(\vx, 0),&&\boldsymbol x\in \Omega,
\end{align}
\end{subequations}
where $u$ is the quantity of interest, $k=1/d$, $f(\vx, t)=-k|\vx|^2\exp{(0.5|\vx|^2+t)}$, $d=10$, $\Omega$ is a unit ball (i.e., $\Omega=\{|\vx|\le 1\}$), and $g(\vx, t)=\exp{(0.5|\vx|^2+t)}$ which is also the analytical solution to above PDEs.

And the reformulated PDEs are (which is used by the proposed model, HC)
\begin{subequations}
\begin{align}
    \frac{\partial u}{\partial t}&=k\big(\nabla\cdot \vp(\vx,t)\big)+f(\vx, t),&&\boldsymbol x\in \Omega\subset \mathbb{R}^d,t\in(0,1], \\
    \vp(\vx, t)&=\nabla u,&&\vx\in \Omega,t\in(0,1], \\
    \vn(\vx) \cdot \vp(\vx,t) &=g(\vx,t),&&\boldsymbol x\in\partial\Omega,t\in(0,1], \label{eq:high_bc} \\
    u(\vx,0)&=g(\vx, 0),&&\boldsymbol x\in \Omega,
\end{align}
\end{subequations}
where $\vp(\vx,t)$ is the introduced extra field.

\paragraph{Construction of the Ansatz}
The solution to the PDEs is a scalar function $u$ and there is only one boundary $\partial\Omega=\{ |\vx|=1 \}$. We now derive the general solution $\vp^{\partial\Omega}$ with respect to Eq.~\eqref{eq:high_bc}
\begin{equation}
    \vp^{\partial\Omega} = \bm{B}(\vx) \mathrm{NN}^{\partial \Omega} + \vn(\vx)g(\vx,t),
\end{equation}
where $\bm{B}(\vx) = \bm{I}_d -\vn(\vx)\vn(\vx)^\top$. And the ansatz $(\hat{u},\hat{\vp})$ is given by
\begin{subequations}
\begin{align}
    \hat{\vp} &= l^{\partial \Omega}(\vx) \mathrm{NN}^{\mathrm{main}}_j(\vx,t)[1:d] + \vp^{\partial\Omega}(\vx,t),\\
    \hat{u} &= \mathrm{NN}^{\mathrm{main}}_j(\vx,t)[d+1] \big(1-\exp{[-\beta_t t]}\big) + g(\vx, 0) \exp{[-\beta_t t]},
\end{align}
\end{subequations}
where $[1:d]$ means taking the first $d$ elements of the output of $\mathrm{NN}^{\mathrm{main}}_j(\vx,t)$ while $[d+1]$ means the last element.

\paragraph{Choices of Extended Distance Functions}
Since $\partial \Omega$ is a ND sphere, we can compute $l^{\partial \Omega}(\vx)$ by subtracting the distance between $\vx$ and the center from the radius (the symbol is different from the previous 2D circles, since $\partial \Omega$ is the outer boundary).

\paragraph{Implementation} 
All the models are trained for 5000 Adam iterations (with a learning rate scheduler of \texttt{ReduceLROnPlateau} from PyTorch and an initial learning rate of 0.01), followed by a L-BFGS optimization until convergence. And the hyper-parameters of each model are listed as follow
\begin{itemize}
    \item \textbf{HC}: The main neural network is an MLP of size $[11] + 4\times[50] + [11]$. The sub-network (corresponding to Eq.~\eqref{eq:high_bc}) is an MLP of size $[11] + 3\times[20] + [10]$. And the hyper-parameters of ``hardness'' is $\beta_t=10$ (here we only have one boundary, so we can construct our ansatz (in the spatial domain) as in Eq.~(3) instead of Eq.~\eqref{eq:hc_time} and $\beta_s$ is no longer needed).
    \item \textbf{PINN}: The ansatz is an MLP of size $[11] + 4 \times [50] + [1]$.
    \item \textbf{PINN-LA}: The parameter of the moving average is $\alpha=0.1$.
    \item \textbf{PINN-LA-2}: The parameter of the moving average is $\alpha=0.1$.
    \item \textbf{PFNN}: The size of the neural network (an MLP) is $[11] + 4\times[50] + [1]$.
\end{itemize}

\subsubsection{Ablation Study: Extra fields}\label{sec_a8_4}
Here, our experiment is divided into two parts where we consider the Poisson's equation and the nonlinear Schrödinger equation, respectively. The relevant details are as follows

\paragraph{Poisson's Equation}
The governing PDEs are described as
\begin{subequations}
\begin{align}
    \Delta u(x)&=-a^2\sin{ax},&&x\in(0,2\pi),\\
    u(x)&=0,&&x=0\lor x=2\pi,
\end{align}
\end{subequations}
where $a=2$ and $u(x)$ is the physical quantity of interest.

And the reformulated PDEs are (corresponding to the \textit{extra fields})
\begin{subequations}
\begin{align}
    \nabla p(x)&=-a^2\sin{ax},&&x\in(0,2\pi),\\
    p(x)&=\nabla u(x),&&x\in(0,2\pi),\\
    u(x)&=0,&&x=0\lor x=2\pi.
\end{align}
\end{subequations}
where $p(x)$ is the introduced extra field.

The two models (PINNs with and without the \textit{extra fields}) are trained with $N_f = 128$ collocation points and $N_b = 2$ boundary points for 10,000 Adam iterations (with a learning rate of 0.001). And we have tested different network architectures, including $[1] + 3\times[50] + [\cdot]$, $[1] + 3\times[100] + [\cdot]$, $[1] + 5\times[50] + [\cdot]$, $[1] + 5\times[100] + [\cdot]$, where for the PINN without the \textit{extra fields}, the number of outputs is 1, and for the PINN with the \textit{extra fields}, the number of outputs is 2. 

\begin{figure}[t]
    \centering
    \subfigure[Poisson's equation]{\label{fig:cv_poission}\includegraphics[width=.40\textwidth]{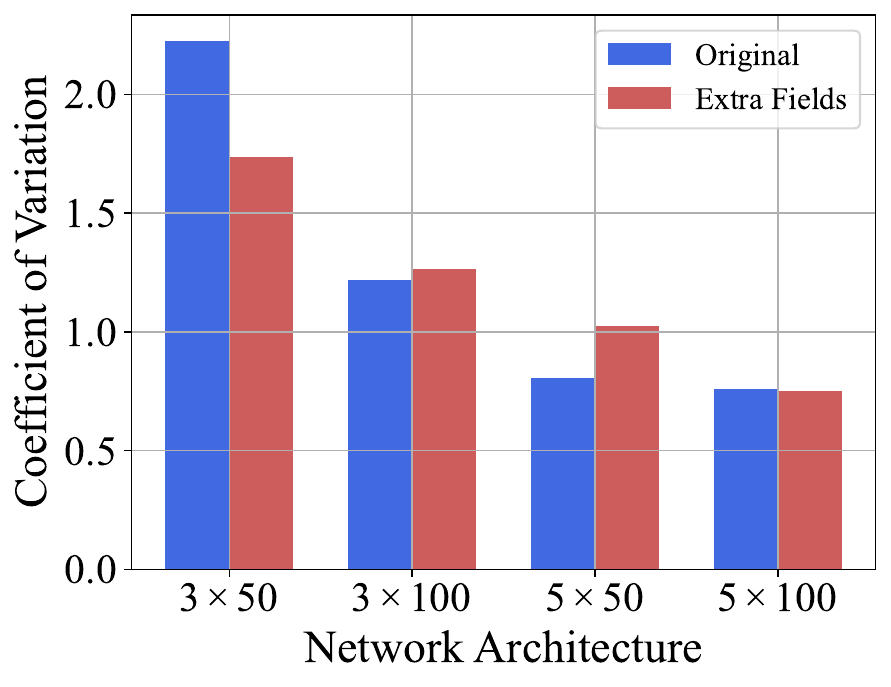}}
    \hspace{1em}
    \subfigure[Schrödinger equation]{\label{fig:cv_sh}\includegraphics[width=.40\textwidth]{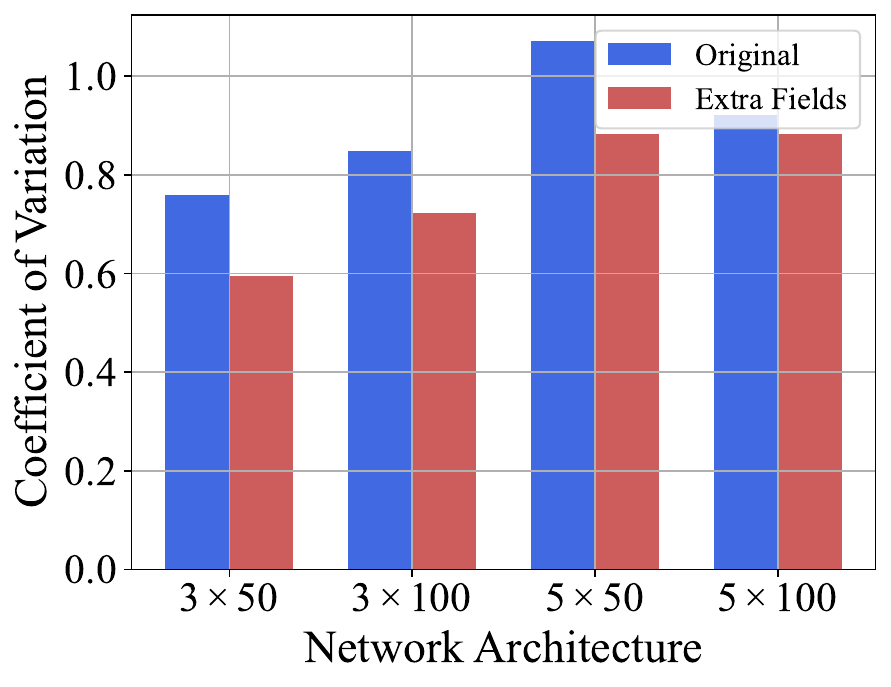}}
    \vspace{-0.1in}
    \caption{The CV of all the values of $\overline{|\nabla_{\bm{\theta}}\mathcal{L}_{\mathcal{F}}|}$ and $\overline{|{\nabla_{\tilde{\bm{\theta}}}\tilde{\mathcal{L}}_{\mathcal{F}}}|}$ during training.}
    \label{fig:abla_cv}
    \vspace{-0.1in}
\end{figure}

\paragraph{Schrödinger Equation}
The governing PDEs are described as
\begin{subequations}
\begin{align}
    i\frac{\partial h}{\partial t}+\frac{1}{2}\frac{\partial^2 h}{\partial x^2}+|h(x,t)|^2h(x,t)&=0,&&x\in(-5,5),t\in(0,\pi/2],\\
    h(t,-5)&=h(t,5),&&t\in(0,\pi/2],\\
    \frac{\partial h}{\partial x}(t,-5)&=\frac{\partial h}{\partial x}(t,5),&&t\in(0,\pi/2],\\
    h(0,x)&=2\sech (x),&&x\in(-5,5),
\end{align}
\end{subequations}
where $h(x,t)$ is the physical quantity of interest.

And the reformulated PDEs are (corresponding to the \textit{extra fields})
\begin{subequations}
\begin{align}
    i\frac{\partial h}{\partial t}+\frac{1}{2}\frac{\partial p}{\partial x}+|h(x,t)|^2h(x,t)&=0,&&x\in(-5,5),t\in(0,\pi/2],\\
    p(x,t)&=\frac{\partial h}{\partial x},&&x\in[-5,5],t\in(0,\pi/2],\\
    h(t,-5)&=h(t,5),&&t\in(0,\pi/2],\\
    \frac{\partial h}{\partial x}(t,-5)&=\frac{\partial h}{\partial x}(t,5),&&t\in(0,\pi/2],\\
    h(0,x)&=2\sech (x),&&x\in(-5,5),
\end{align}
\end{subequations}
where $p(x)$ is the introduced extra field.

The two models (PINNs with and without the \textit{extra fields}) are trained with $N_f = 1000$ collocation points, $N_b = 20$ boundary points, and $N_i = 200$ initial points for 10,000 Adam iterations (with a learning rate of 0.001). And we have tested different network architectures, including $[1] + 3\times[50] + [\cdot]$, $[1] + 3\times[100] + [\cdot]$, $[1] + 5\times[50] + [\cdot]$, $[1] + 5\times[100] + [\cdot]$, where for the PINN without the \textit{extra fields}, the number of outputs is 2, and for the PINN with the \textit{extra fields}, the number of outputs is 4.

\paragraph{Experimental Results}
We report the ratio of the the moving variance (MovVar) of $\overline{|\nabla_{\bm{\theta}}\mathcal{L}_{\mathcal{F}}|}$ 
to that of $\overline{|{\nabla_{\tilde{\bm{\theta}}}\tilde{\mathcal{L}}_{\mathcal{F}}}|}$ at each iteration during training, where the window size of the MovVar is 500 and after the MovVar, a moving average filter with a window size of 500 is applied. The results are shown in Figure~4. Besides, we also calculate the coefficient of variation (CV) of all the values of $\overline{|\nabla_{\bm{\theta}}\mathcal{L}_{\mathcal{F}}|}$ and $\overline{|{\nabla_{\tilde{\bm{\theta}}}\tilde{\mathcal{L}}_{\mathcal{F}}}|}$, respectively. And we give the results in Figure~\ref{fig:abla_cv}. Using the 
CV as a criterion, we also find that the \textit{extra fields} significantly reduces the gradient oscillations during training, especially for the complex nonlinear PDEs.

\subsection{Empirical Analysis of Convergence}\label{sec_a9}
In this subsection, we will empirically analyze the convergence of our method as well as some representative baselines in the context of the simulation of a 2D battery pack (see Section~5.2). We now report the training history with respect to iterations in Figure~\ref{fig:convergence}. The left axis shows the loss of PINNs (including PINN, PINN-LA, and PINN-LA-2) while the right axis shows the loss of our method, HC. The PINN loss is computed by adding up all the loss terms (including the losses of PDEs and BCs), where the loss weights are ignored for PINN-LA and PINN-LA-2. The first 5000 iterations are trained with Adam (separated by the gray dotted line), and the last 15000 are trained with L-BFGS. 

\begin{figure}[t!]
    \centering
    \includegraphics[width=0.6\linewidth]{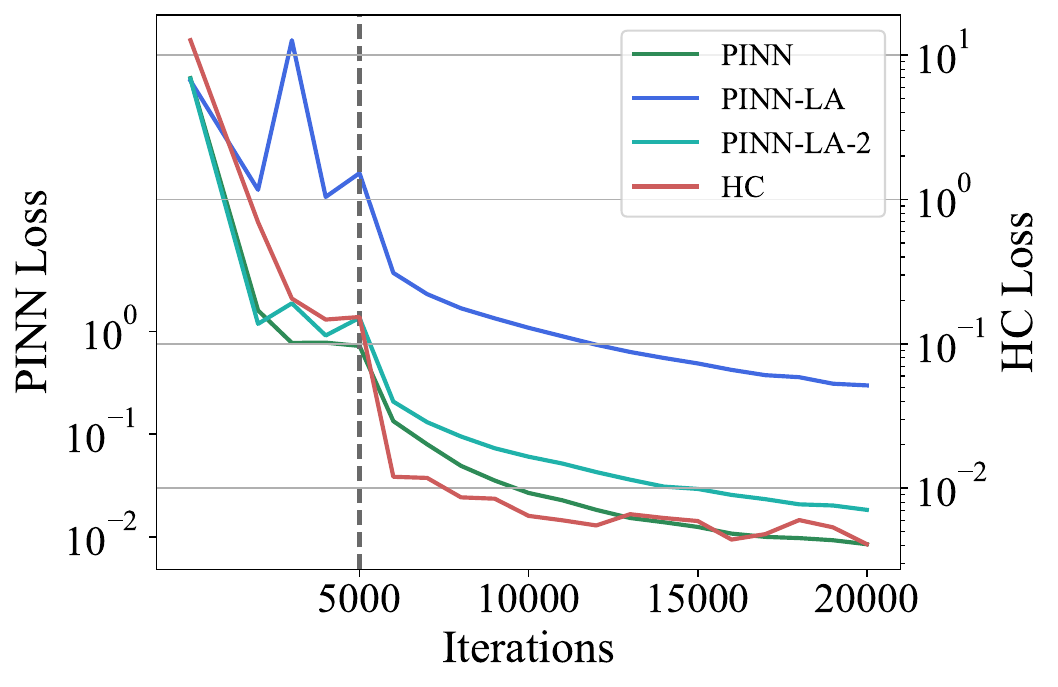}
    \vspace{-0.1in}
    \caption{The convergence history of the simulation of a 2D battery pack.}
    \label{fig:convergence}
    \vspace{-0.1in}
\end{figure}

\begin{table}[b!]
    \vspace{-0.1in}
    \caption{Parallel experimental results of the simulation of a 2D battery pack (MAE of $T$)}
    \vspace{0.1in}
    \centering
    \begin{small}
    \begin{tabular}{lllll}
    \toprule
    & $t=0$     & $t=0.5$     & $t=1$ & average\\
    \midrule
    PINN & $0.1232\pm 0.0219$     & $0.0417 \pm 0.0141$ & $0.0263 \pm 0.0078$ & $0.0499 \pm 0.0135$\\
    PINN-LA & $0.1083\pm 0.0266$     & $0.0927 \pm 0.0372$  & $0.1168\pm 0.0739$ & $0.0969\pm 0.0385$\\
    PINN-LA-2 & $0.1065\pm 0.0059$     & $0.0322\pm 0.0031$  & $\textbf{0.0200}$ $\pm \ 0.0020$ & $0.0400\pm 0.0031$\\
    FBPINN & $0.0763\pm 0.0071$     & $0.0258\pm 0.0037$& $0.0205\pm 0.0041$ & $0.0318\pm 0.0027$\\
    xPINN & $0.2085\pm 0.0252$     & $0.1144\pm 0.0194$  & $0.1352\pm 0.0241$ & $0.1310\pm 0.0194$\\
    PFNN & $\textbf{0.0000}$ $\pm \ 0.0000$  & $0.3769\pm 0.0974$  & $0.6012\pm 0.2274$ & $0.3522\pm 0.1019$\\
    PFNN-2 & $\textbf{0.0000}$ $\pm \ 0.0000$    & $0.3814\pm 0.0381$  & $0.5247\pm 0.0394$ & $0.3365\pm 0.0236$\\
    \midrule
    HC & $\textbf{0.0000}$ $\pm \ 0.0000$ & $\textbf{0.0244}$ $\pm \ 0.0010$  & $0.0226\pm 0.0012$ & $\textbf{0.0219}$ $\pm \ 0.0007$\\
    \bottomrule
    \end{tabular}
    \end{small}
    \label{parallel_tab_pack_1}
\end{table}

From the results in Figure~\ref{fig:convergence}, we can see that the loss functions of all models drop significantly after switching to L-BFGS. This shows that L-BFGS can further promote convergence through utilizing the information of the second derivatives of the loss function. However, we may not start with L-BFGS because it can easily lead to divergence. We consider Adam+L-BFGS to be a practical choice. Furthermore, we find that the convergence of PINNs is negatively affected by the tricks of learning rate annealing algorithm, especially the PINN-LA without our modification. The HC has the fastest convergence rate among all models. This means that the hard-constraint method or extra fields may be helpful in accelerating convergence.

\subsection{Parallel Experiments}\label{sec_a10}
In this subsection, we revisit the three experiments in Section~5.2$\sim$5.4 and perform parallel tests in 5 runs to assess the significance of the results. We report the testing results (along with the $95\%$ confidence intervals) in Table~\ref{parallel_tab_pack_1}$\sim$\ref{parallel_high_2}. From the results, we can see that our method, HC still outperforms all the other baselines. Besides, HC has the least variation, which shows that the hard-constraint methods can improve the stability of training.

\begin{table}[t]
    \vspace{-0.1in}
    \caption{Parallel experimental results of the simulation of a 2D battery pack (MAPE of $T$)}
    \vspace{0.1in}
    \centering
    \begin{small}
    \begin{tabular}{lllll}
    \toprule
    & $t=0$     & $t=0.5$     & $t=1$ & average\\
    \midrule
    PINN & $123.16 \pm 21.91\%$ & $10.97 \pm 3.57\%$  & $4.52\pm 0.98\%$ & $23.58 \pm 5.61\%$\\
    PINN-LA & $108.14 \pm 26.57\%$ & $24.15 \pm 8.07\%$  & $17.98\pm 8.97\%$ & $33.64 \pm 8.89\%$\\
    PINN-LA-2 & $106.39 \pm 5.86\%$ & $8.70 \pm 0.57\%$  & $\textbf{3.86}$ $\pm \ 0.36\%$ & $19.56 \pm 1.00\%$\\
    FBPINN & $76.25 \pm 7.06\%$ & $7.69 \pm 0.68\%$  & $5.26\pm 0.71\%$ & $15.07 \pm 0.57\%$\\
    xPINN & $208.36 \pm 25.21\%$ & $26.25 \pm 4.99\%$  & $18.15\pm 3.05\%$ & $49.60 \pm 7.37\%$\\
    PFNN & $\textbf{0.02}$ $ \pm \ 0.00\%$ & $94.63 \pm 17.68\%$  & $105.38\pm 27.16\%$ & $80.92 \pm 15.08\%$\\
    PFNN-2 & $\textbf{0.02}$ $ \pm \ 0.00\%$ & $71.39 \pm 6.98\%$  & $82.04\pm 8.90\%$ & $61.65 \pm 4.64\%$\\
    \midrule
    HC & $\textbf{0.02}$ $ \pm \ 0.00\%$ & $\textbf{5.29}$ $ \pm \ 0.16\%$  & $3.87\pm 0.14\%$ & $\textbf{5.22}$ $ \pm \ 0.17\%$\\
    \bottomrule
    \end{tabular}
    \end{small}
    \label{parallel_tab_pack_2}
\end{table}

\begin{table}[t]
    \caption{Parallel experimental results of the simulation of an airfoil (MAE)}
    \centering
    \begin{small}
    \begin{tabular}{llll}
    \toprule
    & $u_1$     & $u_2$     & $p$ \\
    \midrule
    PINN & $0.4234\pm 0.0809$     & $0.0681\pm 0.0162$ & $0.3204\pm 0.1404$\\
    PINN-LA & $0.4467\pm 0.0450$     & $0.0630\pm 0.0061$& $0.3028\pm 0.0480$ \\
    PINN-LA-2 & $0.4542\pm 0.0875$     & $0.0679\pm 0.0111$  & $0.3230\pm 0.1115$ \\
    FBPINN & $0.3975\pm 0.0221$     & $0.0544\pm 0.0030$& $0.2650\pm 0.0059$\\
    xPINN & $0.6942\pm 0.0432$     & $0.0581\pm 0.0013$& $1.1587\pm 0.1251$ \\
    \midrule
    HC & $\textbf{0.2824}$ $\pm\  0.0215$ & $\textbf{0.0435}$  $\pm\  0.0024$ & $\textbf{0.2144}$ $\pm\  0.0114$\\
    \bottomrule
    \end{tabular}
    \end{small}
    \label{tab_airfoil_parallel}
    \vspace{-0.1in}
\end{table}

\begin{table}[b]
    \vspace{0.1in}
    \caption{Parallel experimental results of the simulation of an airfoil (WMAPE)}
    \centering
    \begin{small}
    \begin{tabular}{llll}
    \toprule
    & $u_1$     & $u_2$     & $p$ \\
    \midrule
    PINN & $0.5358\pm 0.1024$  & $1.1709\pm 0.2778$ & $0.2921\pm 0.1279$\\
    PINN-LA & $0.5653\pm 0.0570$  & $1.0819\pm 0.1048$ & $0.2760\pm 0.0437$\\
    PINN-LA-2 & $0.5747\pm 0.1106$  & $1.1670\pm 0.1920$ & $0.2944\pm 0.1016$\\
    FBPINN & $0.5030\pm 0.0279$  & $0.9347\pm 0.0517$ & $0.2416\pm 0.0054$\\
    xPINN & $0.8784\pm 0.0546$  & $0.9986\pm 0.0225$ & $1.0562\pm 0.1140$\\
    \midrule
    HC & $\textbf{0.3573}$ $\pm\  0.0272$  & $\textbf{0.7472}$  $\pm\  0.0418$& $\textbf{0.1954}$ $\pm\  0.0104$\\
    \bottomrule
    \end{tabular}
    \end{small}
    \label{tab_airfoil_parallel_2}
    \vspace{-0.1in}
\end{table}

\begin{table}[htbp]
    \vspace{-0.1in}
    \caption{Parallel experimental results of the high-dimensional heat equation (MAE of $u$)}
    \vspace{0.1in}
    \centering
    \begin{small}
    \begin{tabular}{lllll}
    \toprule
    & $t=0$     & $t=0.5$     & $t=1$ & average\\
    \midrule
    PINN & $0.0204\pm 0.0148$     & $0.0357 \pm 0.0104$ & $0.1600 \pm 0.0600$ & $0.0525 \pm 0.0173$\\
    PINN-LA & $0.0430\pm 0.0751$     & $0.3039 \pm 0.6691$  & $0.8011\pm 1.7228$ & $0.3464\pm 0.7531$\\
    PINN-LA-2 & $0.0287\pm 0.0670$     & $0.2071\pm 0.6524$  & $0.5933\pm 1.7225$ & $0.2455\pm 0.7433$\\
    PFNN & $\textbf{0.0000}$ $\pm \ 0.0000$  & $0.0895\pm 0.0727$  & $0.2130\pm 0.1790$ & $0.0963\pm 0.0788$\\
    \midrule
    HC & $\textbf{0.0000}$ $\pm \ 0.0000$ & $\textbf{0.0028}$ $\pm \ 0.0006$  & $\textbf{0.0046}$ $\pm \ 0.0008$ & $\textbf{0.0027}$ $\pm \ 0.0006$\\
    \bottomrule
    \end{tabular}
    \end{small}
    \label{parallel_high_1}
\end{table}

\begin{table}[htbp]
    \caption{Parallel experimental results of the high-dimensional heat equation (MAPE of $u$)}
    \vspace{0.1in}
    \centering
    \begin{small}
    \begin{tabular}{lllll}
    \toprule
    & $t=0$     & $t=0.5$     & $t=1$ & average\\
    \midrule
    PINN & $1.15 \pm 0.51\%$ & $1.41 \pm 0.41\%$  & $3.83\pm 1.45\%$ & $1.75 \pm 0.58\%$\\
    PINN-LA & $2.86 \pm 5.07\%$ & $12.06 \pm 26.54\%$  & $19.35\pm 41.65\%$ & $11.46 \pm 24.75\%$\\
    PINN-LA-2 & $1.92 \pm 4.55\%$ & $8.19 \pm 25.80\%$  & $14.29\pm 41.54\%$ & $8.00 \pm 24.25\%$\\
    PFNN & $\textbf{0.00}$ $ \pm \ 0.00\%$ & $3.59 \pm 2.93\%$  & $5.19\pm 4.36\%$ & $3.20 \pm 2.60\%$\\
    \midrule
    HC & $\textbf{0.00}$ $ \pm \ 0.00\%$ & $\textbf{0.11}$ $ \pm \ 0.03\%$  & $\textbf{0.11}$ $\pm \ 0.02\%$ & $\textbf{0.10}$ $ \pm \ 0.02\%$\\
    \bottomrule
    \end{tabular}
    \end{small}
    \label{parallel_high_2}
\end{table}

\subsection{Ethics Statement}\label{sec_a11}
PDEs have important applications in many fields, including applied physics, automobile manufacturing, economics, and the aerospace industry. Solving PDE via neural networks has attracted much attention in recent years, and it may be applied to the above fields in the future.
Our method also belongs to this kind. However, the method for solving PDE based on neural networks has no theoretical explanation or safety guarantee for the time being. Applying such methods to security-sensitive domains may lead to unexpected incidents and the cause of the accident may be hard to diagnose. Possible solutions include developing alternatives with theoretical interpretability or using safeguards.

\bibliographystyle{plain}
\bibliography{neurips_2022}

\end{document}